\documentclass[svgnames, dvipsnames]{article} 
\usepackage{iclr2025_conference,times}

\usepackage{amsmath,amsfonts,bm}

\def\eqref#1{equation~\ref{#1}}

\def\1{\bm{1}}

\DeclareMathAlphabet{\mathsfit}{\encodingdefault}{\sfdefault}{m}{sl}
\SetMathAlphabet{\mathsfit}{bold}{\encodingdefault}{\sfdefault}{bx}{n}

\usepackage{hyperref}
\usepackage{url}

\usepackage{array}
\usepackage{cleveref}
\usepackage{amsmath}
\usepackage{amssymb}
\usepackage{mathtools}
\usepackage{amsthm}
\usepackage{wrapfig}
\usepackage{booktabs}

\usepackage{bbm}
\usepackage{tcolorbox}
\usepackage{tikz, tikz-cd, pgfplots} 

\DeclareMathOperator{\reff}{ref}
\DeclareMathOperator{\pinn}{PINN}

\newcommand{\fX}{\overrightarrow{X}}

\newcommand{\EE}{\mathbb{E}}
\newcommand{\RR}{\mathbb{R}}
\newcommand{\PP}{\mathbb{P}}
\newcommand{\QQ}{\mathbb{Q}}

\newcommand{\bX}{\overleftarrow{X}}
\newcommand{\bQ}{\overleftarrow{Q}}
\newcommand{\bP}{\overleftarrow{P}}

\newcommand{\bd}{\overleftarrow{\d}}

\DeclareMathOperator{\diag}{diag}

\DeclarePairedDelimiterX{\infdivx}[2]{(}{)}{%
  #1\;\delimsize\|\;#2%
}
\newcommand{\dkl}{D_{\text{KL}}\infdivx}

\DeclarePairedDelimiter{\abs}{\lvert}{\rvert}

\usepackage{etoolbox}
\letcs\originald{\encodingdefault\string\d}
\DeclareRobustCommand*\d
  {\ifmmode\mathop{}\!\mathrm{d}\else\expandafter\originald\fi}

\theoremstyle{plain}
\newtheorem{theorem}{Theorem}[section]

\newtheorem{corollary}[theorem]{Corollary}
\theoremstyle{definition}

\newtheorem{idea}[theorem]{Key Idea}
\theoremstyle{remark}

\title{Continuously Tempered Diffusion Samplers}

\author{Ezra Erives, Bowen Jing, Peter Holderrieth \& Tommi Jaakkola \\
CSAIL, Massachusetts Institute of Technology\\
\texttt{\{erives,bjing,phold\}@mit.edu,\,tommi@csail.mit.edu}  \\
\And
}

\iclrfinalcopy 
\begin{document}

\maketitle
\begin{abstract}
Annealing-based neural samplers seek to amortize sampling from unnormalized distributions by training neural networks to transport a family of densities interpolating from source to target. A crucial design choice in the training phase of such samplers is the \textit{proposal distribution} by which locations are generated at which to evaluate the loss. Previous work has obtained such a proposal distribution by combining a partially learned transport with annealed Langevin dynamics. However, isolated modes and other pathological properties of the annealing path imply that such proposals achieve insufficient exploration and thereby lower performance post training. To remedy this, we propose \textit{continuously tempered diffusion samplers}, which leverage exploration techniques developed in the context of molecular dynamics to improve proposal distributions. Specifically, a family of distributions across different temperatures is introduced to lower energy barriers at higher temperatures and drive exploration at the lower temperature of interest. We empirically validate improved sampler performance driven by extended exploration. Code is available at \href{https://github.com/eje24/ctds}{https://github.com/eje24/ctds}.
\end{abstract}

\section{Introduction}
A challenging task in Bayesian statistics and natural sciences is to sample from distributions of the form
\begin{equation}
    \label{eq:pi_and_pi_hat}
    \pi(x) = \frac{1}{Z}\hat{\pi}(x),
\end{equation}
where $\hat{\pi}(x)$ denotes a known unnormalized density and $Z = \int_{\RR^d}\hat{\pi}(x)\, \d x$ denotes the unknown \textit{partition function}. Traditional MCMC-based sampling approaches, including Metropolis-Hastings and Hamiltonian Monte-Carlo (HMC), have been observed to suffer from pseudo-ergodic behavior for high-dimensional or otherwise complex choices of the density $\hat\pi$. Techniques such as annealed importance sampling (AIS) \citep{ais}, and sequential Monte-Carlo-based approaches \citep{smc} work instead over an annealed sequence of densities $\{\hat{\pi}_t\}$ which interpolates a tractable source $\hat{\pi}_0$ and the desired target $\hat{\pi}_1 = \hat \pi$. Recently, \emph{neural samplers} have been proposed which leverage a wide array of learned transport techniques, parameterized with neural networks, to optimize the annealed auxiliary distribution \citep{mcvae, snf, uha_geffner, uha_zhang, mcd, ldvi, cmcd, improved_sampling, nets}.


In this work, we consider neural samplers which learn a \emph{control} or \emph{transport} to generate the evolution of a prespecified annealing path $\hat\pi_t$ via a so-called physics-informed neural network (PINN) loss, which penalizes the degree to which the learned vector field fails to satisfy the  continuity equation \citep{pinn, pinn_workshop, boltzmann_interpolation, nets}. The off-policy nature of the PINN loss means that a crucial design choice in PINN-based sampler training is the \textit{proposal distribution} $\tilde{\pi}_t$ by which samples are generated at which to evaluate the PINN loss. To this end, recent work has proposed to construct $\tilde{\pi}_t$ as the marginals of annealed Langevin dynamics steered with the partially-learned control. By augmenting the learned control in this manner, samples are driven towards high density regions of the density path, allowing for boostrapped training in which exploration gradually improves with the quality of the learned control. However, for many choices of density path, including the commonly used linear interpolation (see \cref{eq:linear_path}), isolated modes, high energy barriers, and other problematic features of the energy landscape obstruct this exploration process, resulting in poor sampler performance \cite{boltzmann_interpolation}. 

\begin{figure}[!t]
    \centering
    \includegraphics[width=\textwidth]{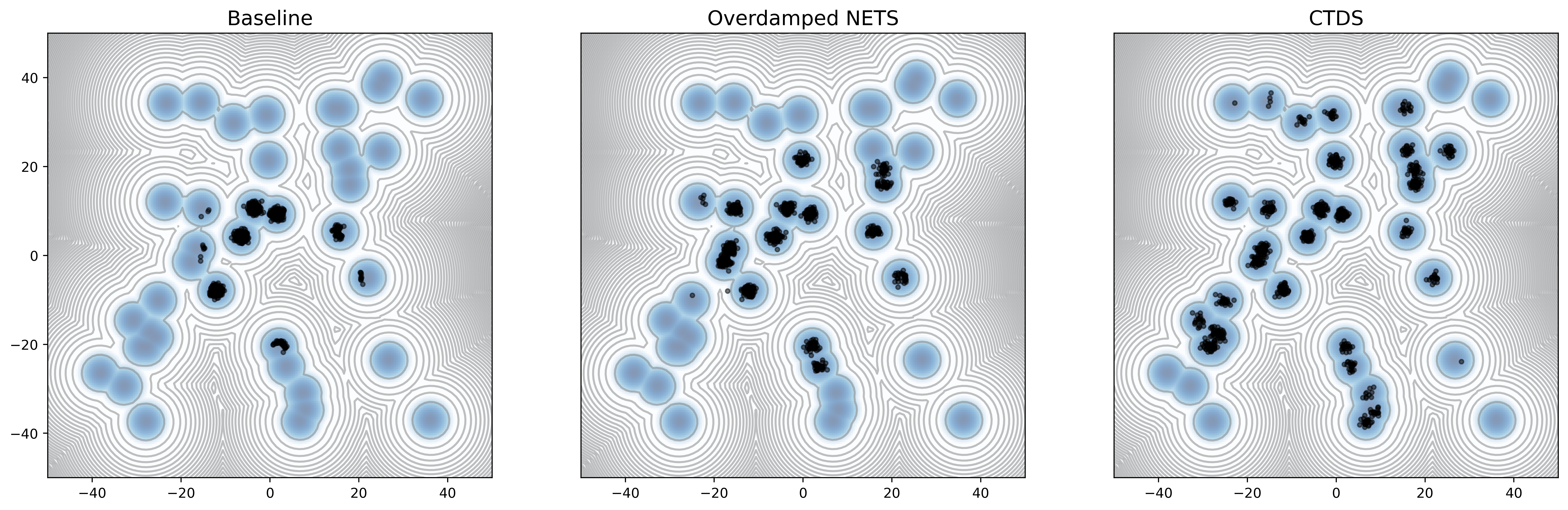}
    \caption{A comparison of samples generated by simulating the ODE given by a control $\mu_t^\theta(x_t)$ learned with various choices of proposal for the 40-mode Gaussian mixture in \cref{fab_experiment}. Left: Trained with a baseline proposal obtained by simulating the control with no added Langevin dynamics. Center: Trained with a Jarzynski-reweighted controlled overdamped Langevin proposal as in \citep{nets}. Right: Trained with a Jarzynski-reweighted CTDS proposal.}
    \label{fig:experiment_scatter}
\end{figure}
To alleviate this problem, we turn toward \textit{tempering techniques}, which are widely used in e.g., molecular dynamics to overcome high energy barriers between separate metastable states. Their success is based primarily on 
the insight that such barriers are lowered when the temperature is raised \citep{simulated_tempering, parallel_tempering_1986, parallel_tempering_1997, parallel_tempering_1999, md_continuous_tempering}. Specifically, we draw inspiration from recently proposed \textit{continuous tempering} approaches, in which Langevin dynamics is run over an augmented state space given by both position and a continuous temperature variable \cite{continuous_tempering, continuously_tempered_hmc, thermostat_enabled}. We propose to use such a continuously tempered dynamics as the basis for the proposal distribution, in doing so exploiting lower energy barriers at higher temperatures to drive exploration and thus better training. Specifically, we make the following contributions:

\begin{enumerate}
    \item We extend the notion of a time-indexed density path to a time-and-temperature-indexed \textit{density continuum}. We propose a multi-temperature PINN objective over such a continuum which allows for sampler training to be amortized \textit{across temperatures}.
    \item To learn dynamics across the time-temperature continuum, we propose the \textit{continuously tempered diffusion sampler (CTDS)}, a scheme that allows for efficient exploration during sampler training by leveraging continuous tempering techniques.
    \item We validate the performance of CTDS-trained samplers against those trained with existing proposals and demonstrate improved sampler performance.
\end{enumerate}

\section{Background}
\paragraph{Notation}We will denote normalized and unnormalized probability densities by $\pi$ and $\hat{\pi}$ respectively. Further, we will call $F = - \log Z$ the \emph{free energy} and $U = -\log \hat{\pi}$ the \emph{energy}, so that $\pi(x) = \frac{1}{Z}\hat{\pi}(x) = e^{-U(x) + F}$. Additionally, we will denote by $\mathcal{T}$ the unit-time interval $[0,1]$, and by $\mathcal{P}(\mathcal{X})$ the space of twice-differentiable, fully-supported, normalized densities on the space $\mathcal{X}$. When we work over phase space coordinates, we will denote the momentum as $p$. 
\subsection{Controlled Langevin Annealing}
\label{subsec:controlled_langevin_annealing}
To overcome pseudo-ergodic behavior in traditional MCMC-based sampling approaches \citep{metropolis_hastings, hmc}, techniques such as annealed importance sampling (AIS) \cite{ais}, and later sequential Monte Carlo (SMC) based approaches \cite{smc}, work by simulating MCMC over an annealed \textit{density path}
\begin{equation}
    \label{eq:density_path}
    \{\hat{\pi}_t\}_{t \in \mathcal{T}}, \quad \quad \hat{\pi}_0 \propto \pi_{\text{source}},\,\quad \text{and}\,\quad \hat{\pi}_1 \propto \pi_{\text{target}},
\end{equation}
interpolating between $\pi_{\text{source}}$ and $\pi_{\text{target}}$. 
Recent work has since emphasized combining the sampling process with a learned transformation using e.g., normalizing flows \citep{snf, aft, craft}, and, in the continuous time setting, a learned vector field \citep{cmcd, lfs, path_guided_particle_sampler, nets}. Following the second approach, and given a prescribed density path as in \cref{eq:density_path}, we may consider a controlled (underdamped) Langevin dynamics
\begin{equation}
    \label{eq:controlled_underdamped_langevin}
    \begin{aligned}
        dX_t &= \mu_t^{\theta}(X_t) + \gamma_t M^{-1}P_t\d t\\ dP_t &= \gamma_t\left[\nabla \log \pi_t(X_t) - \varepsilon_t M^{-1}P_t\right]\d t + \sqrt{2 \gamma_t \varepsilon_t}\d W_t
    \end{aligned}
\end{equation}
as in \citep{zhong_asymmetric, nets}, as well as the overdamped limit given by
\begin{equation}
    \label{eq:controlled_overdamped_langevin}
    dX_t = \left[\mu_t^{\theta}(X_t) + \varepsilon_t \nabla \log p_t(X_t)\right] \d t + \sqrt{2\varepsilon_t}\d W_t,
\end{equation}
as in \citet{nets, cmcd}, where in both cases $\mu_t^\theta$ is the learned, time-dependent control.\footnote{The underdamped parameterization is slightly different than analogous formulations presented in \citet{ldvi, cmcd}.} Here, $W_t$ denotes a standard Brownian motion on $\mathbb{R}^d$, ($d = d_x$ being the dimensionality of the data), and $\gamma_t$, $\varepsilon_t$, and $M$ denote scaling and damping coefficients, and particle weight, respectively. The advantage of the controlled Langevin schemes in \cref{eq:controlled_underdamped_langevin} and \cref{eq:controlled_overdamped_langevin} is that the Langevin dynamics can be seen as correcting for an imperfectly learned control, and thus allowing for a bootstrapped training procedure in which the marginals of the forward process converge on the desired density path as the control improves \citep{nets}. In particular, perfect sampling is recovered as when $\gamma_t \to \infty$ in \cref{eq:controlled_underdamped_langevin} and $\varepsilon_t \to \infty$ in \cref{eq:controlled_overdamped_langevin}, or when the the control is learned perfectly so that it satisfies the continuity equation 
\begin{equation}
    \label{eq:continuity_eq}
    \partial_t \pi_t(x_t) = - \nabla \cdot \left[\mu_t^{\theta}(x_t)\pi_t(x_t)\right],\quad\quad\quad \forall (x_t,t) \in \mathbb{R}^d \times \mathcal{T},
\end{equation}
and thus facilitates an instantaneous transition along the density path for \textit{any} choice of $\gamma_t$ or $\epsilon_t$. Several popular families of training objective have emerged to learn such a control, including minimizing a path KL divergence \citep{cmcd, pis, dis, dds}, log-variance objectives \citep{og_logvar, nusken_hjb, cmcd, improved_sampling}, and physics-informed neural network (PINN) losses \citep{pinn, expert_guide_pinn, pinn_workshop}. The last of these is the focus of this work and will be discussed more below.

The success of controlled annealing approaches depends heavily on the choice of density path, for which the most common choice is linear path 
\begin{equation}
    \label{eq:linear_path}
    U_t^{\text{linear}}(x) \triangleq (1-t)U_0(x) + tU_1(x).
\end{equation}
While nearly universally available, the linear path has been shown to suffer from certain regularity issues including the so-called teleportation of mass phenomenon described in \citep{boltzmann_interpolation}, in certain cases causing any control satisfying \cref{eq:continuity_eq} to exhibit singularities making it hard to integrate numerically, let alone learn. One approach toward alleviating these issues is proposed in \citet{boltzmann_interpolation}, in which an additional learned term is added to the linear path from \cref{eq:linear_path} to obtain
\begin{equation}
    \label{eq:learned_path}
    U_t^{\text{linear}}(x) \triangleq (1-t)U_0(x) + tU_1(x) + t(1-t)U_t^\theta(x),
\end{equation}
which we shall refer to as the \textit{learned path}.

\subsection{PINN-Based Objectives}
In the sampler setting, PINN-based objectives seek to learn the control $\mu_t^{\theta}$ by minimizing the expected pointwise residual error of the continuity equation \cref{eq:continuity_eq}. While such an objective, as originally formulated, requires access to the generally intractable free energy $F_t$, the free energy may be jointly learned along with the control, yielding the modified PINN objective
\begin{equation}
    \label{eq:pinn}
    \mathcal{L}_{\text{PINN}}(\mu_t^{\theta},F_t^{\theta}; \tilde{\pi}_t) \triangleq \int_{\mathcal{T} \times \mathbb{R}^d}\lvert  \partial_t F^{\theta}_t - \partial_t U_t(x) + \nabla_x \cdot \mu_t^\theta(x) - (\nabla_x U_t(x))^T \mu_t^\theta(x) \rvert^2\,\tilde{\pi}_t(x) \d x \d t
\end{equation}
for which the true free energy $F_t$ can be shown to be the unique minimizer \citep{boltzmann_interpolation, nets}. More details on the PINN loss, including a derivation, can be found in \cref{appendix:pinn}. Besides the choice of density path $\pi_t$, the fundamental object underpinning \cref{eq:pinn} is the \textit{proposal distribution} $\tilde{\pi}_t$, by which samples are drawn at which to evaluate the objective. A good choice of proposal distribution $\tilde{\pi}_t$ is given by the density path itself (i.e., when $\tilde{\pi}_t = \pi_t$), but such a choice is nearly always intractable. As an alternative, recent work has proposed to construct such a proposal distribution using a controlled annealed Langevin scheme \citep{nets}.
\begin{tcolorbox}[colback=gray!10,colframe=gray!0,arc=0mm]
\begin{idea}[Proposal via controlled Langevin annealing]
    Construct the PINN proposal $\tilde{\pi}$ as the marginals of a controlled Langevin annealing process, as is given in \cref{eq:controlled_underdamped_langevin} and \cref{eq:controlled_overdamped_langevin}. Intuitively, as the quality of the learned control increases, so will the mode-coverage of the proposal. \textit{Optionally}, one may additionally reweigh samples toward the ground-truth marginal $\pi$ using a Jarzynski-like equality \citep{nets}.
\end{idea}
\end{tcolorbox}


\subsection{Tempering Techniques for Enhanced Exploration}
\label{subsec:tempering}
One family of techniques to improving exploration are tempering-based methods, based on the fundamental insight that energy barriers are lowered and mixing times are improved when one \textit{raises the temperature} of the target to obtain $\pi_t^{\beta}(x) \propto \pi_t(x)^{\beta}$, for inverse temperature $\beta \triangleq \frac{1}{T} < 1$. To exploit this insight, approaches such as parallel tempering \citet{parallel_tempering_1986, parallel_tempering_1997, parallel_tempering_1999} and simulated tempering \citet{simulated_tempering} have been proposed. These are MCMC-based sampling schemes over a temperature ladder along with mechanisms by which samples may jump between temperatures, thereby allowing faster-mixing dynamics at higher temperatures to drive convergence at a lower temperature of interest.

More recently, continuous-tempering schemes have been proposed involving a Hamiltonian dynamics over a temperature-augmented extended state thereby allowing for a smoothly-varying temperature variable \cite{continuous_tempering, continuously_tempered_hmc, thermostat_enabled}, and demonstrating improvements over simulated and parallel tempering \cite{continuously_tempered_hmc}. Such continuous tempering schemes have found utility in MD simulations \citet{md_continuous_tempering}, Bayesian posterior inference \citet{thermostat_enabled}, and in improving the convergence of piece-wise deterministic Markov processes \citet{ct_pdmp}. In this work, we extend the application of continuous tempering to the annealed neural sampler setting.

\section{Continuously Tempered Diffusion Samplers}
In this section, we present our main contribution: combining the controlled Langevin annealing schemes of \cref{subsec:controlled_langevin_annealing} with the continuous tempering approaches outlined in \cref{subsec:tempering}, obtaining a new, tempered class of annealing schemes over a continuously-varying temperature variable, which we refer to as continuously tempered diffusion samplers (CTDS).

\subsection{Tempered Density Continuums}
\label{subsec:density_continuum}
Recall that a density path is a $t$-indexed annealing path $\hat{\pi}_t(x) \triangleq e^{-U_t(x)}$ which agrees with a prescribed source density $\hat{\pi}_0$ and target density $\hat{\pi}_1$. It is then natural that we extend the notion of density path to the tempered setting by affixing an additional indexing variable - the inverse-temperature variable $\beta$ - to obtain the annealing continuum
 \begin{equation}
     \label{eq:density_continuum}
     \hat{\pi}_t^{\beta}(x) \triangleq e^{-U_t^{\beta}(x)},\quad \quad U_t^{\beta}(x) \propto \beta U_t^1(x) \quad \quad \quad \forall (t,b) \in \mathcal{T} \times \mathcal{B},
 \end{equation}
 where we have defined $\mathcal{B} \triangleq [\beta_{\text{min}}, 1]$ as the range of inverse-temperatures $\beta$ of interest. We refer to the resulting continuously-indexed collection of densities $\{\hat{\pi}_t^{\beta}(x)\}_{t,\beta \in \mathcal{T} \times \mathcal{B}}$ as a \emph{density continuum}. Given a target density $U^{\text{target}}$ of interest, two such density continuums are the \emph{linear continuum} and the \emph{learned continuum} defined respectively for a prescribed target density $U_1^1(x) = U^{\text{target}}(x)$ by $\pi_0^{\beta}(x) = \mathcal{N}(0,\frac{1}{\beta}I_d)$ (thereby defining $U_0^{\beta}$) and
 \begin{align}
     U_t^{\beta}(x) &= (1-t)U_0^{\beta}(x) + tU_1^{\beta}(x) && \blacktriangleright \text{linear density continuum}\label{eq:linear_continuum}\\
     U_t^{\beta}(x) &= (1-t)U_0^{\beta}(x) + tU_1^{\beta}(x) + \beta t (1-t)U_t^{\theta}(x)  && \blacktriangleright \text{learned density continuum}\label{eq:learned_continuum}
 \end{align}
 for all $(t,b) \in \mathcal{T} \times \mathcal{B}$. In the next two sections, we will construct an annealed Langevin dynamics over such a tempered density continuum. Doing so directly over $\mathcal{B}$ presents two challenges: First, it is difficult to handle the boundaries of $\mathcal{B}$ (at $\beta_{\text{min}}$ and $1.0$) without introducing some sort of boundary condition on the Langevin dynamics, or with the introduction of some confining potential, neither of which are particularly desirable. Second, samples from any continuous marginal density over $\beta$ would almost surely not be equal to either $\beta=\beta_{\text{min}}$ (thereby discouraging mixing behavior at high temperatures) or $\beta=1.0$ (thereby not providing training signal at the original temperature of interest). We therefore follow the lead of \citet{continuous_tempering} and reparameterize the inverse-temperature coordinate by pulling back to the space $\Xi \triangleq \mathbb{R}$ via a carefully defined continuous mapping $\beta(\xi): \Xi \to \mathcal{B}$ in a manner which resolves both issues; see \cref{appendix:reparameterization} for details.

\subsection{PINN Objectives Across Temperature}
\label{subsec:ct_pinn}

For fixed $\xi=\xi^\star$ (the case of $\beta(\xi^\star) = 1$ being of particular interest), we recover the standard annealed sampling problem, in which we would like to learn a control $\mu_t^\theta(x, \xi^\star)$ which satisfies the continuity equation
\begin{equation}
    \label{eq:fixed_z_continuity_eq}
    \partial_t \pi_t^{\xi^\star}(x) = \nabla_x \cdot [\mu_t^\theta(x,\xi^\star)\pi_t^{\xi^\star}(x)],
\end{equation}
corresponding to $\{\hat{\pi}_t^{\xi^\star}(x)\}_{t \in \mathcal{T}}$ (and in turn, to the path $\{\hat{\pi}_t^{\beta(\xi^\star)}(x)\}_{t \in \mathcal{T}}$). To obtain such a control, let us consider a multi-temperature variant of the PINN loss over the space $\mathbb{R}^d \times \mathcal{T} \times \Xi$ from \cref{eq:pinn}, given by
\begin{equation}
\label{eq:ct_pinn}
\mathcal{L}^{\text{MT}}_{\pinn}(F^{\theta},\mu^{\theta}; \tilde{\pi}) \triangleq \int_{\mathcal{T} \times \Xi \times \mathbb{R}^d} \lvert  \partial_t F_t^{\theta} - \partial_t U_t^{\xi} + \nabla_x \cdot \mu_t^{\theta} - (\nabla_x U_t^{\xi})^T \mu_t^{\theta} \rvert^2\,\tilde{\pi}_t^{\xi}(x,\xi) \d x \d \xi \d t,
\end{equation}
for free energy estimate $F: \mathcal{T} \times \Xi \to \mathbb{R}$, control $\mu: \mathcal{T}\times \mathbb{R}^d \times \Xi \to \mathbb{R}^d$, and proposal distribution $\tilde{\pi}_t(x,\xi)\in \mathcal{P}(\mathbb{R}^d)$. Now, defining the free energy $F_t(\xi)$ as $F_t(\xi) \triangleq \int_{\mathbb{R}^d} e^{- U_t^{\xi}(x)}\,\d x$, we propose the following result which characterizing the minimizing $F$ of \cref{eq:ct_pinn}.

\begin{tcolorbox}[colback=gray!10,colframe=gray!0,arc=0mm]
\begin{theorem}
\label{thm:ct_pinn_characterization}
    Let $\tilde{\pi}_t(x, \xi) \in \mathcal{P}(\mathbb{R}^d \times \Xi)$ and suppose there exists a pair $(F^{\star}, \mu^{\star})$ which satisfies
    \begin{equation*}
        \mathcal{L}_{\pinn}(F^{\star},\mu^{\star}; \tilde{\pi}) = 0,
    \end{equation*}
    as well as the boundary condition $F_0^\star(\xi) = F_0(\xi)$ for all $\xi \in \Xi$. In this case, $F^{\star}$ is unique and is given by the free energy $F_t(\xi)$. 
\end{theorem}
\end{tcolorbox}
The result from \cref{thm:ct_pinn_characterization} extends an existing result in the single-temperature setting (\cite{boltzmann_interpolation}, Lemma 1) to the multi-temperature setting, and a short proof can be found in \cref{appendix:mt_pinn_existence}. Furthermore, \cref{thm:ct_pinn_characterization} establishes that we may learn the control $\mu_t^\theta(x,\xi)$ by simultaneously learning the free energy $F_t^{\theta}(\xi)$. To construct such a proposal $\tilde{\pi}_t(x,\xi)$, we look to finally exploit the tempered nature of the density continuum $\{\hat{\pi}_t^{\xi}(x)\}$, as we discuss next.

\subsection{Continuously Tempered Diffusion Samplers}
\label{subsec:ctds}
In this section, we will at last formulate our main contribution. First, we will realize the density \emph{continuum} as a density \emph{path} by introducing a fictitious, time-dependent density over $\xi$. Second, we will construct the \emph{continuously tempered diffusion sampler}, a controlled annnealed Langevin dynamics over such a density path. By introducing some time-dependent, \emph{normalized} potential $\psi_t(\xi): \mathcal{T} \times \Xi \to \mathbb{R}$, we may realize the continuum $\{\hat{\pi}_t^{\xi}(x)\}$ as a joint density 
\begin{equation}
    \label{eq:xz_joint_density}
    \hat{\pi}_t(x, \xi) \triangleq e^{-U_t^{\xi}(x) + \psi_t(\xi)} = e^{-\tilde{U}_t(x,\xi)},\quad\quad\quad \forall (t,\xi) \in \mathcal{T} \times \Xi
\end{equation}
over the augmented state $(x,\xi)$, where we have defined $\tilde{U}_t(x,\xi) \triangleq U_t(x,\xi) - \psi_t(\xi)$. Then it is easily seen that the marginal $\pi_t(\xi)$ is given by $ \pi_t(\xi) \propto e^{\psi_t(\xi) - F_t(\xi)}$. As a consequence, we may rewrite $\psi_t^{\theta}(\xi) = \psi_t'(\xi) + F_t^{\theta}(\xi)$, where $F_t^{\theta}(\xi)$ is a learned approximation of the free energy as in \cref{eq:pinn}, and so that as $F_t^{\theta}(\xi) \to F_t(\xi)$, $\pi_t(\xi)$ becomes distributed proportionally to $e^{\psi_t'(\xi)}$. In doing so, we obtain 
\begin{equation}
    \label{eq:xz_theta_joint_density}
    \hat{\pi}_t^{\theta}(x, \xi) = e^{-U_t^{\xi}(x) + F_t^\theta(\xi) + \psi_t'(\xi)} = e^{-\tilde{U}_t^{\theta}(x,\xi)},
\end{equation}
so as to ``cancel-out'' the bias of the unknown free energy $F_t(\xi)$ as the quality of our learned approximation $F_t^{\theta}(\xi)$ increases. In \cref{appendix:biasing_and_confining_potentials}, we characterize more explicitly the construction of $\psi_t'(\xi)$ in practice, and in particularly the inclusion of a confining potential term. Denoting by $q_t = (x_t, \xi_t)$ and $p_t = (p_t^x, p_t^{\xi})$ we may now define the \textit{non-separable} Hamiltonian
\begin{equation}
    \mathcal{H}_t^{\theta}(q_t,p_t) \triangleq \tilde{U}_t^{\theta}(q_t) + K(q_t, p_t), \quad \quad K(q_t,p_t) \triangleq \frac{\beta(\xi_t)}{2M_x} \lVert p_t^x \rVert^2 + \frac{1}{2M_{\xi}} \lVert p_t^{\xi} \rVert^2.
\end{equation}
With the inclusion of the learned control $\mu_t^\theta$, we obtain the controlled annealed underdamped Langevin dynamics given by
\begin{equation}
\label{eq:ctds}
\begin{aligned}
\d X_t &= \mu_t^{\theta}(X_t,\xi_t) +\gamma_t^x \nabla_{p^x} K(q_t, p_t) \d t\\
\d \xi_t &= \gamma_t^{\xi} \nabla_{p^{\xi}} K(q_t, p_t)\\
\d P^x_t &= \gamma_t^x\left[-\nabla_{x} \mathcal{H}_t^{\theta}(q_t,p_t) - \varepsilon_t^x \nabla_{p^x} K(q_t,p_t)\right]\d t + \sqrt{2 \gamma_t^{x} \varepsilon_t^x}\d W_t^{x}\\
\d P_t^{\xi} &= \gamma_t^{\xi}\left[-\nabla_{\xi} \mathcal{H}_t^{\theta}(q_t,p_t) - \varepsilon_t^{\xi} \nabla_{p^{\xi}} K(q_t,p_t)\right]\d t + \sqrt{2 \gamma_t^{\xi} \varepsilon_t^{\xi}}\d W_t^{\xi}\\
(X_0, P_0^x, \xi_0, P_0^{\xi}) &\sim \pi_0^{\dagger}(x, \xi, p^x, p^{\xi}).
\end{aligned}
\end{equation}
where the extended joint density $\pi_t^{\dagger}$ is given by
\begin{equation}
    \label{eq:pi_dagger}
    \pi_t^\dagger(x_t, \xi_t, p_t^x, p_t^{\xi}) = \pi_t^{\theta}(x_t, \xi_t)\mathcal{N}(p_t^x;0, \frac{M_x}{\beta(\xi_t)}I_d)\mathcal{N}(p_t^{\xi};0, M_{\xi}) \propto e^{-\mathcal{H}_t^{\theta}(q_t,p_t)}.
\end{equation}
In \cref{eq:ctds}, we have used $(\gamma_t^x, \gamma_t^{\xi})$, $(\varepsilon_t^x, \varepsilon_t^{\xi})$, and $(M_x, M_{\xi})$ to denote the respective scaling, damping, and mass coefficients for each of $x$ and $\xi$, and have denoted by $W_t^x$ and $W_t^{\xi}$ standard Brownian motions on $\mathbb{R}^d$ and $\Xi$ respectively. We refer to this scheme as a \textit{continuously tempered diffusion sampler} (CTDS). We may then construct the proposal $\tilde{\pi}_t^{\xi}$ as the marginals of the forward process given by \cref{eq:ctds}.

\subsection{A Continuously Tempered Controlled Jarzynski Result}
\label{subsec:crooks_and_jarzynski}
For imperfectly learned control $\mu_t^\theta(x, \beta(\xi))$ and finite scaling coefficients $\gamma_t^{x}$ and $\gamma_t^{\xi}$, the proposal distribution $\tilde{\pi}_t^{\beta}$ induced by the marginals of the CTDS forward process in \cref{eq:ctds} will lag behind the true density path $\pi_t(x, \xi)$. To correct for this discrepancy, we derive a controlled Jarzynski equality, as in \citep{cmcd,nets}, so that we may reweigh samples from \cref{eq:ctds} to the correct joint density $\pi_t^{\theta}(x,\xi)$ described in \cref{eq:xz_theta_joint_density}. Specifically, for fixed $T \in \mathcal{T}$, we may consider the work functional as in \citep{vaikuntanathan2008escorted, zhong_asymmetric, nets}, given by
\begin{equation}
        \label{eq:work}
        A_T(Q) \triangleq \int_0^T \nabla_x \cdot \mu_t^{\theta}(x_t, \xi_t) - \partial_t \log \tilde{U}^{\theta}_t(x_t, \xi_t) - \mu_t^{\theta}(x_t,\xi_t)^T \nabla_x \tilde{U}_t^{\theta}(x_t, \xi_t) \d t,
\end{equation}
where we have denoted $q_t = (x_t, \xi_t)$ and $Q = \{q_t\}_{t \in [0,T]}$. Defining $\PP_T$ as the path measure of the forward process \cref{eq:ctds} on the interval $[0,T] \subseteq \mathcal{T}$, and letting $F_t$ denote the free energy of the density path $\hat{\pi}_t(x, \xi)$, we have the following controlled Jarzynski identity.
\begin{tcolorbox}[colback=gray!10,colframe=gray!0,arc=0mm]
\begin{theorem}[Continuously Tempered Controlled Jarzynski Equality]
\label{thm:ct_jarzynski}
For $T \in (0,1]$, let $h: \mathbb{R}^d \times \Xi \to \mathbb{R}$ denote some observable. Then 
\begin{equation}
    \EE_{(x,\xi) \sim \pi_t(x, \xi))} \left[h(x,\xi)\right] = \frac{\EE_{Q \sim \PP_T} \left[ h(Q)\exp\left(A_T(Q)\right)\right]}{\EE_{Q \sim \PP_T} \left[ \exp\left(A_T(Q)\right)\right]}.
\end{equation}
In particular, 
\begin{equation}
    \EE_{(x,\xi) \sim \pi_t(x, \xi))} \left[\exp(A_T(Q)) \right] = \exp(F_0 - F_T),
    \label{eq:jarzynski}
\end{equation}
where we have defined $F_t \triangleq \int_{\Xi \times \mathbb{R}^d} e^{-\tilde{U}_t^{\theta}(x,z)}\, \d x \d z$.
\end{theorem}
\end{tcolorbox}
We note that \cref{thm:ct_jarzynski} closely resembles recently established controlled Jarzynski results from \citep{cmcd, nets}, that it is in fact a corollary of a more general controlled variant of the Crooks fluctuation theorem as established in \citep{vaikuntanathan2008escorted, cmcd, zhong_asymmetric}, and that this resemblance is due entirely to the CTDS construction as an annealed Langevin dynamics. We present a self-contained proof of \cref{thm:ct_jarzynski} by way of an analogous Crooks identity in \cref{appendix:crooks_and_jarzynski_section}. We conclude by demonstrating how \cref{thm:ct_jarzynski} can be used to reweigh samples from the forward process \cref{eq:ctds} to match the joint density $\pi_t^\theta(x,\xi)$ defined in \cref{eq:xz_theta_joint_density}. This also allows us to reweight the loss and up-weigh ``important'' samples during training. Observe that by \cref{thm:ct_jarzynski}, it follows that

\begin{equation}
    \label{eq:reweighted_pinn}
\begin{aligned}
    \mathcal{L}^{\text{MT}}_{\pinn}(F,\mu; \pi^\theta) &=  \int_{\mathcal{T}} \frac{\mathbb{E}_{Q \sim \PP_t}\left[\lvert  \partial_t F_t - \partial_t U_t + \nabla_x \cdot \mu_t - (\nabla_x U_t)^T \mu_t \rvert^2\exp(A_t(Q))\right]}{\mathbb{E}_{Q \sim \PP_t} \left[\exp(A_t(Q))\right]} \d t.
\end{aligned}
\end{equation}
\begin{wrapfigure}[16]{r}{0.5\textwidth}
    \centering
    \includegraphics[width=\linewidth]{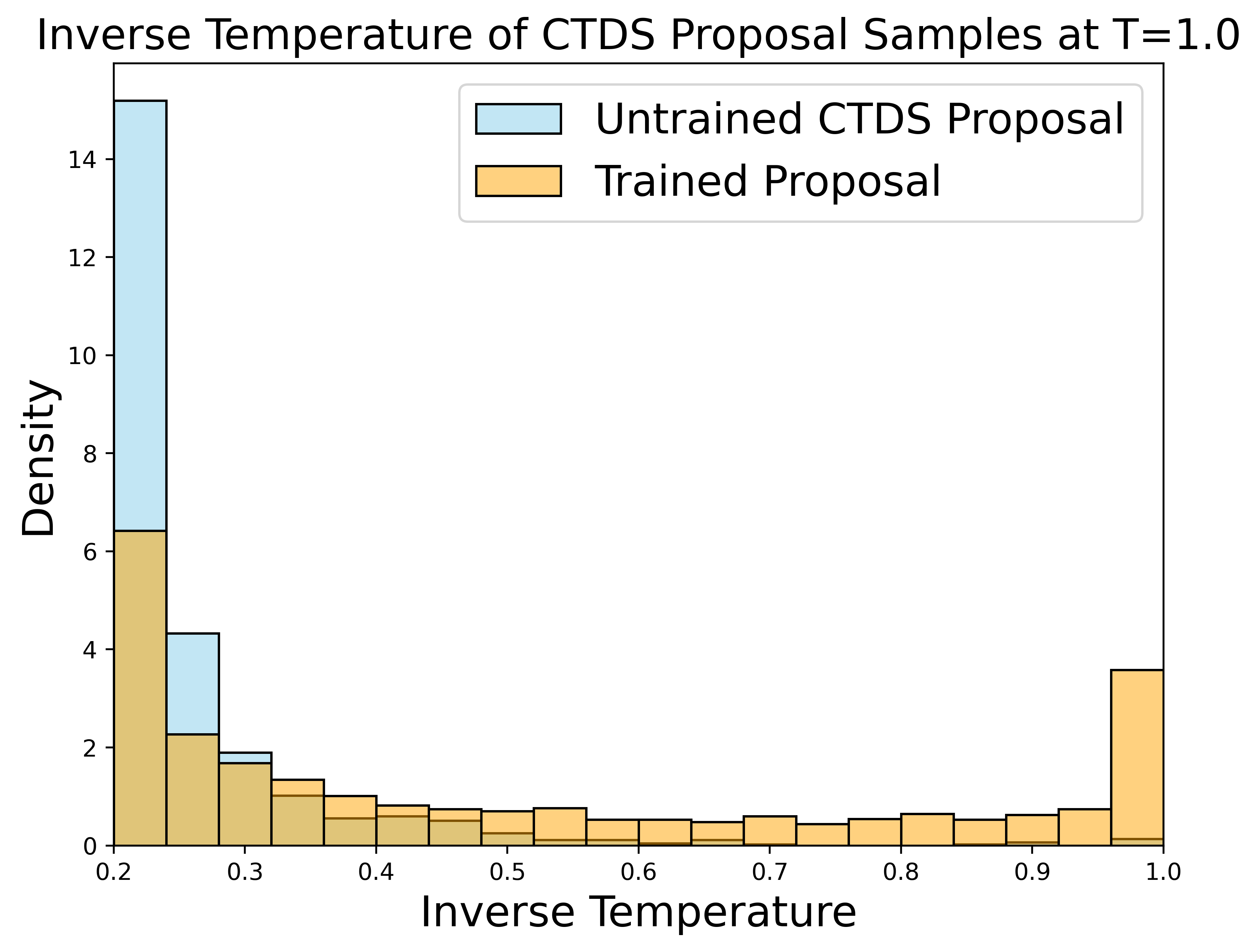}
    \vspace{-0.6cm}
    \caption{Distribution of sample inverse temperatures at $t=1.0$ obtained by simulating \cref{eq:ctds} before and after training for the non-Jarzynski-reweighted CTDS proposal as in \cref{fab_experiment}.}
    \label{fig:temperature_histogram}
\end{wrapfigure}

Details can be found in \cref{appendix:ct_specialization}. We conclude from \cref{eq:reweighted_pinn} that we may reweight samples from the forward process \cref{eq:ctds} so as obtain the proposal $\pi^\theta_t(x,\xi)$, which as we recall from \cref{eq:xz_theta_joint_density}, converges to $\pi_t^{\xi}(x)\pi_t(\xi) = \pi_t^{\xi}(x)e^{-\psi'(\xi)}$ as $F_t^{\theta}(\xi) \to F_t(\xi)$. At the start of training, or when $F_t^\theta(\xi) \approx F_t(\xi)$ otherwise fails to hold, it is observed that samples tend to ``pool'' at either higher or lower temperatures (depending on the target). This effect diminishes as the quality of the free energy approximation increases; see \cref{fig:temperature_histogram}. It should also be noted that since \cref{thm:ct_jarzynski} reweights toward the joint distribution $\pi_t^\theta(x, \xi)$, a poor estimate of the free energy will bias the reweighting toward particular values of $\xi$.

\section{Experiments}

\subsection{40-Mode Gaussian Mixture} 
\label{fab_experiment}
To investigate the benefit of the CTDS proposal, we compare the performance of samplers trained with various choices of proposal on a modified version challenging 40-mode Gaussian mixture from \citep{fab}.\footnote{We take the standard deviation of the isotropic mixture components to be $\sigma = 0.25$.} We emphasize upfront, and to the strongest degree possible, that great care must be taken in comparing the results we report to existing work. Different works consider different density paths and various degrees of complexity for this problem. For example, learning a sampler for such a target \emph{can be made artificially easier by e.g., increasing the scale of the source distribution $\pi_0$}. In this experiment, we choose to take our source distribution $\pi_0 = \mathcal{N}(0,5.0I_2)$, thereby ensuring that modes discovered during training are done so on account of the proposal, rather than an advantageously chosen initialization. As our density path, we take the learned path given in \cref{eq:learned_path} (and for CTDS, the multi-temperature analog given by \cref{eq:learned_continuum}).

\paragraph{Training.} As a baseline, we consider the proposal obtained by simulating the control by itself with no Langevin dynamics, as in \citep{boltzmann_interpolation}, and which we label as the reference proposal. We additionally consider proposals obtained from the controlled overdamped (\cref{eq:controlled_overdamped_langevin}, with $\varepsilon_t = 50.0$) and controlled underdamped/inertial (\cref{eq:controlled_underdamped_langevin}, with $(\gamma_t,\varepsilon_t = 50.0, 2.0)$) Langevin annealing dynamics as in \citep{nets}, which we train with and without use of the Jarzynski equality. Finally, we consider our CTDS proposal (\cref{eq:ctds}, with $(\gamma_t^x,\gamma_t^{\xi}, \varepsilon_t^x, \varepsilon_t^{\xi}) = (50.0, 2.0, 5.0, 2.0)$) both with and without reweighting with the Jarzynski equality (see \cref{subsec:crooks_and_jarzynski}). We train using a replay buffer, re-sampling once per epoch for 1250 epochs for a total of 125000 training iterations at which the PINN loss is evaluated at 6250 randomly sampled elements of the buffer. The control, learnable density path, and free energy are all parameterized as simple MLPs. Integration is performed with an Euler-Maruyama solver with step size $\Delta t = 0.002$. We additionally utilized Gaussian Fourier features as in \cite{tancik2020fourierfeaturesletnetworks} to encode spatial, temporal, and temperature variables, and which we found to improve training stability and sampler performance. More details can be found in \cref{appendix:experiment}

\begin{figure}[t!]
        \centering
        \renewcommand{\arraystretch}{1.0}
        \begin{tabular}{lccccc}
            \toprule
            \multicolumn{4}{c}{$40$-mode Gaussian Mixture} \\
            \midrule
            \textbf{Proposal}  & $\mathcal{W}_2$ $\downarrow$  & \textbf{ELBO} $\uparrow$ & \textbf{EUBO} $\downarrow$ \\
            \midrule
            Baseline   & $23.90 \pm 0.32$ & $-2.11 \pm 0.00$ & $23.78 \pm 0.34$ \\
            NETS (OD) w/o Jar.   & $22.61 \pm 0.16$ & $-1.65 \pm 0.00$ & $22.27 \pm 0.23$ \\
            NETS (OD)    & $20.11\pm0.23$& $-1.15\pm 0.08$ & $19.77 \pm 0.22$ \\
            NETS (UD) w/o Jar.  & $23.47\pm 0.31$ & $-1.98\pm 0.00$ & $22.00 \pm 0.32$ \\
            NETS (UD) & $21.30 \pm 0.23$ & $-1.35 \pm 0.02$ & $ 20.95 \pm 0.26$ \\
            CTDS w/o Jar. (ours)  & $\mathbf{12.87\pm 0.20}$ & $\mathbf{-0.24 \pm 0.01}$ & $ 19.91 \pm 0.45$ \\
            CTDS (ours)  & $ 14.70 \pm 0.28 $ & $ -0.29 \pm 0.01$ & $\mathbf{15.35 \pm 0.34}$ \\
            \bottomrule
        \end{tabular}
    \caption{Results for samplers trained on the 40-mode Gaussian mixture target as in \cref{fab_experiment}.}
    \label{fig:results_table}
\end{figure}

\paragraph{Evaluation.} Post-training, we sample by simulating the dynamics
\begin{equation}
    \label{eq:experiment_ode}
    \d X_t = \mu_t^{\theta}(X_t) \d t, \quad \quad X_0 \sim \pi_0 \triangleq \mathcal{N}(0, 5.0I_2).
\end{equation}
In the CTDS case, we take $\mu_t^{\theta}(X_t) = \mu_t^{\theta}(X_t, \beta(\xi) = 1.0)$. We report mean and std. values for both 2-Wasserstein ($\mathcal{W}_2$), evidence lower bound (ELBO), and evidence upper bound (EUBO), for each using $N = 2500$ samples for each of $10$ trials, as can be found in \cref{fig:results_table}. We observe that the CTDS outperform their baseline, overdamped and underdamped counterparts, achieving both lower $\mathcal{W}_2$ values and ELBOs, and that the Jarzynski-reweighted CTDS achieves the lowest EUBO. Curiously, while reweighted NETS based proposals seem to uniformly outperform non-reweighted counterparts, the same trend does not hold for CTDS, and in particular the non-reweighted CTDS proposal achieves both lower $\mathcal{W}_2$ and higher ELBO. We hypothesize that this is due in part to the effects discussed in \cref{subsec:crooks_and_jarzynski} (and visualized in \cref{fig:temperature_histogram}), by which an initially poor free energy approximation biases both samples and importance weights are initially biased towards high temperatures. A visualization comparing the baseline, reweighted overdamped NETS, and reweighted CTDS is provided in \cref{fig:experiment_scatter}.

\section{Conclusion}
In this work, we have proposed \textit{continuously tempered diffusion samplers} (CTDS), a novel family of controlled Langevin annealing processes based on continuous tempering techniques from e.g., molecular dynamics. When combined with PINN-based sampler training, CTDS-based proposals outperform existing proposals based on overdamped and underdamped Langevin annealing on the challenging 40-mode Gaussian mixture from \citet{fab}. We leave it to future work to extend CTDS to higher dimensional or otherwise complex targets and other scientific applications.

\section{Acknowledgements}
We thank Jiajun He, Yuanqi Du, and Michael Albergo for helpful discussions regarding PINN-based sampler training over the course of this project. We additionally acknowledge support from the Machine Learning for Pharmaceutical Discovery and Synthesis (MLPDS) consortium, the DTRA Discovery of Medical Countermeasures Against New and Emerging (DOMANE) threats program, and the NSF Expeditions grant (award 1918839) Understanding the World Through Code

\bibliography{iclr2025_conference}

\begin{thebibliography}{45}
\providecommand{\natexlab}[1]{#1}
\providecommand{\url}[1]{\texttt{#1}}
\expandafter\ifx\csname urlstyle\endcsname\relax
  \providecommand{\doi}[1]{doi: #1}\else
  \providecommand{\doi}{doi: \begingroup \urlstyle{rm}\Url}\fi

\bibitem[Albergo \& Vanden-Eijnden(2024)Albergo and Vanden-Eijnden]{nets}
Michael~S. Albergo and Eric Vanden-Eijnden.
\newblock Nets: A non-equilibrium transport sampler, 2024.
\newblock URL \url{https://arxiv.org/abs/2410.02711}.

\bibitem[Arbel et~al.(2021)Arbel, Matthews, and Doucet]{aft}
Michael Arbel, Alexander G. D.~G. Matthews, and Arnaud Doucet.
\newblock Annealed flow transport monte carlo, 2021.
\newblock URL \url{https://arxiv.org/abs/2102.07501}.

\bibitem[Berner et~al.(2022)Berner, Richter, and Ullrich]{dis}
Julius Berner, Lorenz Richter, and Karen Ullrich.
\newblock An optimal control perspective on diffusion-based generative modeling.
\newblock \emph{arXiv preprint arXiv:2211.01364}, 2022.

\bibitem[Blessing et~al.(2024)Blessing, Jia, Esslinger, Vargas, and Neumann]{beyondelbo}
Denis Blessing, Xiaogang Jia, Johannes Esslinger, Francisco Vargas, and Gerhard Neumann.
\newblock Beyond elbos: A large-scale evaluation of variational methods for sampling, 2024.
\newblock URL \url{https://arxiv.org/abs/2406.07423}.

\bibitem[Crooks(1999)]{Crooks_1999}
Gavin~E. Crooks.
\newblock Entropy production fluctuation theorem and the nonequilibrium work relation for free energy differences.
\newblock \emph{Physical Review E}, 60\penalty0 (3):\penalty0 2721–2726, September 1999.
\newblock ISSN 1095-3787.
\newblock \doi{10.1103/physreve.60.2721}.
\newblock URL \url{http://dx.doi.org/10.1103/PhysRevE.60.2721}.

\bibitem[Del~Moral et~al.(2006)Del~Moral, Doucet, and Jasra]{smc}
Pierre Del~Moral, Arnaud Doucet, and Ajay Jasra.
\newblock Sequential monte carlo samplers.
\newblock \emph{Journal of the Royal Statistical Society Series B: Statistical Methodology}, 68\penalty0 (3):\penalty0 411--436, 2006.

\bibitem[Doucet et~al.(2022)Doucet, Grathwohl, Matthews, and Strathmann]{mcd}
Arnaud Doucet, Will Grathwohl, Alexander G. D.~G. Matthews, and Heiko Strathmann.
\newblock Score-based diffusion meets annealed importance sampling, 2022.
\newblock URL \url{https://arxiv.org/abs/2208.07698}.

\bibitem[Elfwing et~al.(2017)Elfwing, Uchibe, and Doya]{silu}
Stefan Elfwing, Eiji Uchibe, and Kenji Doya.
\newblock Sigmoid-weighted linear units for neural network function approximation in reinforcement learning, 2017.
\newblock URL \url{https://arxiv.org/abs/1702.03118}.

\bibitem[Fan et~al.(2024)Fan, Zhou, Tian, and Qian]{path_guided_particle_sampler}
Mingzhou Fan, Ruida Zhou, Chao Tian, and Xiaoning Qian.
\newblock Path-guided particle-based sampling.
\newblock In \emph{Forty-first International Conference on Machine Learning}, 2024.
\newblock URL \url{https://openreview.net/forum?id=Kt4fwiuKqf}.

\bibitem[Geffner \& Domke(2021)Geffner and Domke]{uha_geffner}
Tomas Geffner and Justin Domke.
\newblock Mcmc variational inference via uncorrected hamiltonian annealing, 2021.
\newblock URL \url{https://arxiv.org/abs/2107.04150}.

\bibitem[Geffner \& Domke(2023)Geffner and Domke]{ldvi}
Tomas Geffner and Justin Domke.
\newblock Langevin diffusion variational inference, 2023.
\newblock URL \url{https://arxiv.org/abs/2208.07743}.

\bibitem[Gobbo \& Leimkuhler(2015)Gobbo and Leimkuhler]{continuous_tempering}
Gianpaolo Gobbo and Benedict~J Leimkuhler.
\newblock Extended hamiltonian approach to continuous tempering.
\newblock \emph{Physical Review E}, 91\penalty0 (6):\penalty0 061301, 2015.

\bibitem[Graham \& Storkey(2017)Graham and Storkey]{continuously_tempered_hmc}
Matthew~M. Graham and Amos~J. Storkey.
\newblock Continuously tempered hamiltonian monte carlo, 2017.
\newblock URL \url{https://arxiv.org/abs/1704.03338}.

\bibitem[Hansmann(1997)]{parallel_tempering_1997}
Ulrich~HE Hansmann.
\newblock Parallel tempering algorithm for conformational studies of biological molecules.
\newblock \emph{Chemical Physics Letters}, 281\penalty0 (1-3):\penalty0 140--150, 1997.

\bibitem[Kingma \& Ba(2017)Kingma and Ba]{adam}
Diederik~P. Kingma and Jimmy Ba.
\newblock Adam: A method for stochastic optimization, 2017.
\newblock URL \url{https://arxiv.org/abs/1412.6980}.

\bibitem[Kunita(2019)]{kunita}
Hiroshi Kunita.
\newblock \emph{Stochastic flows and jump-diffusions}.
\newblock Springer, 2019.

\bibitem[Lenner \& Mathias(2016)Lenner and Mathias]{md_continuous_tempering}
Nicolas Lenner and Gerald Mathias.
\newblock Continuous tempering molecular dynamics: a deterministic approach to simulated tempering.
\newblock \emph{Journal of chemical theory and computation}, 12\penalty0 (2):\penalty0 486--498, 2016.

\bibitem[Liptser \& Shiryaev(2013)Liptser and Shiryaev]{liptser2013statistics}
Robert~S Liptser and Albert~N Shiryaev.
\newblock \emph{Statistics of random processes: I. General theory}, volume~5.
\newblock Springer Science \& Business Media, 2013.

\bibitem[Luo et~al.(2019)Luo, Wang, Yang, Zhu, and Wang]{thermostat_enabled}
Rui Luo, Jianhong Wang, Yaodong Yang, Zhanxing Zhu, and Jun Wang.
\newblock Thermostat-assisted continuously-tempered hamiltonian monte carlo for bayesian learning, 2019.
\newblock URL \url{https://arxiv.org/abs/1711.11511}.

\bibitem[Marinari \& Parisi(1992)Marinari and Parisi]{simulated_tempering}
E~Marinari and G~Parisi.
\newblock Simulated tempering: A new monte carlo scheme.
\newblock \emph{Europhysics Letters (EPL)}, 19\penalty0 (6):\penalty0 451–458, July 1992.
\newblock ISSN 1286-4854.
\newblock \doi{10.1209/0295-5075/19/6/002}.
\newblock URL \url{http://dx.doi.org/10.1209/0295-5075/19/6/002}.

\bibitem[Matthews et~al.(2023)Matthews, Arbel, Rezende, and Doucet]{craft}
Alexander G. D.~G. Matthews, Michael Arbel, Danilo~J. Rezende, and Arnaud Doucet.
\newblock Continual repeated annealed flow transport monte carlo, 2023.
\newblock URL \url{https://arxiv.org/abs/2201.13117}.

\bibitem[Metropolis et~al.(1953)Metropolis, Rosenbluth, Rosenbluth, Teller, and Teller]{metropolis_hastings}
Nicholas Metropolis, Arianna~W Rosenbluth, Marshall~N Rosenbluth, Augusta~H Teller, and Edward Teller.
\newblock Equation of state calculations by fast computing machines.
\newblock \emph{The journal of chemical physics}, 21\penalty0 (6):\penalty0 1087--1092, 1953.

\bibitem[Midgley et~al.(2023)Midgley, Stimper, Simm, Schölkopf, and Hernández-Lobato]{fab}
Laurence~Illing Midgley, Vincent Stimper, Gregor N.~C. Simm, Bernhard Schölkopf, and José~Miguel Hernández-Lobato.
\newblock Flow annealed importance sampling bootstrap, 2023.

\bibitem[Máté \& Fleuret(2023)Máté and Fleuret]{boltzmann_interpolation}
Bálint Máté and François Fleuret.
\newblock Learning interpolations between boltzmann densities, 2023.
\newblock URL \url{https://arxiv.org/abs/2301.07388}.

\bibitem[Neal(1998)]{ais}
Radford~M. Neal.
\newblock Annealed importance sampling, 1998.
\newblock URL \url{https://arxiv.org/abs/physics/9803008}.

\bibitem[Neal(2012)]{hmc}
Radford~M Neal.
\newblock Mcmc using hamiltonian dynamics.
\newblock \emph{arXiv preprint arXiv:1206.1901}, 2012.

\bibitem[Nüsken \& Richter(2023)Nüsken and Richter]{nusken_hjb}
Nikolas Nüsken and Lorenz Richter.
\newblock Solving high-dimensional hamilton-jacobi-bellman pdes using neural networks: perspectives from the theory of controlled diffusions and measures on path space, 2023.
\newblock URL \url{https://arxiv.org/abs/2005.05409}.

\bibitem[Raissi et~al.(2019)Raissi, Perdikaris, and Karniadakis]{pinn}
Maziar Raissi, Paris Perdikaris, and George~E Karniadakis.
\newblock Physics-informed neural networks: A deep learning framework for solving forward and inverse problems involving nonlinear partial differential equations.
\newblock \emph{Journal of Computational physics}, 378:\penalty0 686--707, 2019.

\bibitem[Richter \& Berner(2024)Richter and Berner]{improved_sampling}
Lorenz Richter and Julius Berner.
\newblock Improved sampling via learned diffusions, 2024.
\newblock URL \url{https://arxiv.org/abs/2307.01198}.

\bibitem[Richter et~al.(2020)Richter, Boustati, Nüsken, Ruiz, and Ömer Deniz~Akyildiz]{og_logvar}
Lorenz Richter, Ayman Boustati, Nikolas Nüsken, Francisco J.~R. Ruiz, and Ömer Deniz~Akyildiz.
\newblock Vargrad: A low-variance gradient estimator for variational inference, 2020.
\newblock URL \url{https://arxiv.org/abs/2010.10436}.

\bibitem[Sugita \& Okamoto(1999)Sugita and Okamoto]{parallel_tempering_1999}
Yuji Sugita and Yuko Okamoto.
\newblock Replica-exchange molecular dynamics method for protein folding.
\newblock \emph{Chemical physics letters}, 314\penalty0 (1-2):\penalty0 141--151, 1999.

\bibitem[Sun et~al.()Sun, Berner, Azizzadenesheli, and Anandkumar]{pinn_workshop}
Jingtong Sun, Julius Berner, Kamyar Azizzadenesheli, and Anima Anandkumar.
\newblock Physics-informed neural networks for sampling.
\newblock In \emph{ICLR 2024 Workshop on AI4DifferentialEquations In Science}.

\bibitem[Sutton et~al.(2022)Sutton, Salomone, Chevallier, and Fearnhead]{ct_pdmp}
Matthew Sutton, Robert Salomone, Augustin Chevallier, and Paul Fearnhead.
\newblock Continuously-tempered pdmp samplers, 2022.
\newblock URL \url{https://arxiv.org/abs/2205.09559}.

\bibitem[Swendsen \& Wang(1986)Swendsen and Wang]{parallel_tempering_1986}
Robert~H Swendsen and Jian-Sheng Wang.
\newblock Replica monte carlo simulation of spin-glasses.
\newblock \emph{Physical review letters}, 57\penalty0 (21):\penalty0 2607, 1986.

\bibitem[Tancik et~al.(2020)Tancik, Srinivasan, Mildenhall, Fridovich-Keil, Raghavan, Singhal, Ramamoorthi, Barron, and Ng]{tancik2020fourierfeaturesletnetworks}
Matthew Tancik, Pratul~P. Srinivasan, Ben Mildenhall, Sara Fridovich-Keil, Nithin Raghavan, Utkarsh Singhal, Ravi Ramamoorthi, Jonathan~T. Barron, and Ren Ng.
\newblock Fourier features let networks learn high frequency functions in low dimensional domains, 2020.
\newblock URL \url{https://arxiv.org/abs/2006.10739}.

\bibitem[Thin et~al.(2021)Thin, Kotelevskii, Doucet, Durmus, Moulines, and Panov]{mcvae}
Achille Thin, Nikita Kotelevskii, Arnaud Doucet, Alain Durmus, Eric Moulines, and Maxim Panov.
\newblock Monte carlo variational auto-encoders, 2021.
\newblock URL \url{https://arxiv.org/abs/2106.15921}.

\bibitem[Tian et~al.(2024)Tian, Panda, and Lin]{lfs}
Yifeng Tian, Nishant Panda, and Yen~Ting Lin.
\newblock Liouville flow importance sampler, 2024.
\newblock URL \url{https://arxiv.org/abs/2405.06672}.

\bibitem[Vaikuntanathan \& Jarzynski(2008)Vaikuntanathan and Jarzynski]{vaikuntanathan2008escorted}
Suriyanarayanan Vaikuntanathan and Christopher Jarzynski.
\newblock Escorted free energy simulations: Improving convergence by reducing dissipation.
\newblock \emph{Physical Review Letters}, 100\penalty0 (19):\penalty0 190601, 2008.

\bibitem[Vargas \& N{\"u}sken(2023)Vargas and N{\"u}sken]{cmcd}
Francisco Vargas and Nikolas N{\"u}sken.
\newblock Transport, variational inference and diffusions: with applications to annealed flows and schr$\backslash$" odinger bridges.
\newblock \emph{arXiv preprint arXiv:2307.01050}, 2023.

\bibitem[Vargas et~al.(2023)Vargas, Grathwohl, and Doucet]{dds}
Francisco Vargas, Will Grathwohl, and Arnaud Doucet.
\newblock Denoising diffusion samplers, 2023.
\newblock URL \url{https://arxiv.org/abs/2302.13834}.

\bibitem[Wang et~al.(2023)Wang, Sankaran, Wang, and Perdikaris]{expert_guide_pinn}
Sifan Wang, Shyam Sankaran, Hanwen Wang, and Paris Perdikaris.
\newblock An expert's guide to training physics-informed neural networks, 2023.
\newblock URL \url{https://arxiv.org/abs/2308.08468}.

\bibitem[Wu et~al.(2020)Wu, Köhler, and Noé]{snf}
Hao Wu, Jonas Köhler, and Frank Noé.
\newblock Stochastic normalizing flows, 2020.
\newblock URL \url{https://arxiv.org/abs/2002.06707}.

\bibitem[Zhang et~al.(2021)Zhang, Hsu, Li, Finn, and Grosse]{uha_zhang}
Guodong Zhang, Kyle Hsu, Jianing Li, Chelsea Finn, and Roger~B Grosse.
\newblock Differentiable annealed importance sampling and the perils of gradient noise.
\newblock \emph{Advances in Neural Information Processing Systems}, 34:\penalty0 19398--19410, 2021.

\bibitem[Zhang \& Chen(2022)Zhang and Chen]{pis}
Qinsheng Zhang and Yongxin Chen.
\newblock Path integral sampler: a stochastic control approach for sampling, 2022.
\newblock URL \url{https://arxiv.org/abs/2111.15141}.

\bibitem[Zhong et~al.(2024)Zhong, Kuznets-Speck, and DeWeese]{zhong_asymmetric}
Adrianne Zhong, Ben Kuznets-Speck, and Michael~R. DeWeese.
\newblock Time-asymmetric fluctuation theorem and efficient free-energy estimation.
\newblock \emph{Physical Review E}, 110\penalty0 (3), September 2024.
\newblock ISSN 2470-0053.
\newblock \doi{10.1103/physreve.110.034121}.
\newblock URL \url{http://dx.doi.org/10.1103/PhysRevE.110.034121}.

\end{thebibliography}
\bibliographystyle{iclr2025_conference}

\appendix
\section{A Brief Overview of PINN Objectives}
\label{appendix:pinn}
We briefly derive the PINN objective from \cref{eq:pinn}. Recall the the continuity equation 
\begin{equation}
    \partial_t \pi_t(x_t) = - \nabla \cdot \left[\mu_t^{\theta}(x_t)\pi_t(x_t)\right],\quad\quad\quad \forall (x_t,t) \in \mathbb{R}^d \times \mathcal{T},
\end{equation}
from \cref{eq:continuity_eq}. Dividing both sides by $\pi_t(x_t)$ yields
\begin{equation}
    \partial_t \log \pi_t(x_t) = - \nabla \mu_t^{\theta}(x_t) - \mu_t^{\theta}(x_t)^T \nabla \log \pi_t(x_t).
\end{equation}
Finally, plugging. in $\log \pi_t(x_t) = -U_t(x) + F_t$ and rearranging yields
\begin{equation}
    \partial_t F_t - \partial_t U_t(x) + \nabla \mu_t^{\theta}(x_t) - (\nabla_x U_t(x))^T \mu_t^\theta(x) = 0
\end{equation}
from which we obtain the PINN objective
\begin{equation}
    \mathcal{L}_{\text{PINN}}(\mu_t^{\theta}; \tilde{\pi}_t) \triangleq \int_{\mathcal{T} \times \mathbb{R}^d}\lvert  \partial_t F_t - \partial_t U_t(x) + \nabla_x \cdot \mu_t^\theta(x) - (\nabla_x U_t(x))^T \mu_t^\theta(x) \rvert^2\,\tilde{\pi}_t(x) \d x \d t
\end{equation}
for fully-supported \emph{proposal distribution} $\tilde{\pi}_t$ \citep{pinn, pinn_workshop, expert_guide_pinn}. The free energy $F_t$ however is generally unknown, and thus must be estimated \citep{lfs, path_guided_particle_sampler}, or jointly-learned with the control \cite{boltzmann_interpolation, nets}. In line with this second approach, subsequent work has established a modified PINN objective by which the free energy may be simultaneously learned along with the control, viz.,
\begin{equation}
    \mathcal{L}_{\text{PINN}}(\mu_t^{\theta},F_t^{\theta}; \tilde{\pi}_t) \triangleq \int_{\mathcal{T} \times \mathbb{R}^d}\lvert  \partial_t F^{\theta}_t - \partial_t U_t(x) + \nabla_x \cdot \mu_t^\theta(x) - (\nabla_x U_t(x))^T \mu_t^\theta(x) \rvert^2\,\tilde{\pi}_t(x) \d x \d t
\end{equation}
for which the true free energy $F_t$ can be shown to be the unique minimizer \cite{boltzmann_interpolation, nets}. To better facilitate the learning of the free energy, past work has proposed a curriculum training procedure in which the domain of integration along the time dimension is slowly annealing from zero to one, so as to initially learn the free energy on some smaller interval before all of $\mathcal{T}$.

\section{Extended Details on Continuously Tempered Diffusion Samplers}
In this section, we elaborate an additional design choices for continuously tempered diffusion samplers, including the choice of reparameterization, the use of confining and biasing potentials for temperature. Additionally, we provide a proof of \cref{thm:ct_pinn_characterization}.
\subsection{Temperature Reparameterization Details}
\label{appendix:reparameterization}
Let us recall the definition of a density continuum over $(t,\beta) \in \mathcal{T} \times \mathcal{B}$ frorm \cref{eq:density_continuum}, given by 
 \begin{equation}
     \hat{\pi}_t^{\beta}(x) \triangleq e^{-U_t^{\beta}(x)},\quad \quad U_t^{\beta}(x) \propto \beta U_t^1(x) \quad \quad \quad \forall (t,b) \in \mathcal{T} \times \mathcal{B},
 \end{equation}
 where we have defined $\mathcal{B} \triangleq [\beta_{\text{min}}, 1]$ as the range of inverse-temperatures $\beta$ of interest. Our goal of constructing an annealed Langevin dynamics over such a continuum necessitates a joint density over both $x$ and $\beta$, and in turn, an associated marginal distribution over $\beta$. Working directly over $\beta$ is thus undesirable for two reasons: First, assuming that such a marginal over $\beta$ is continuous, samples from this marginal would almost surely not be equal to either $\beta_{\text{min}}$ or $1.0$, the two temperatures of particular relevance. This is problematic because it would mean that any proposal constructed with e.g., Langevin dynamics over such a joint density would provide little training signal at the desired temperature of $1.0$, nor would it spend much exploring at $t = \beta_{\text{min}}$. Secondly, we would like to constrain the temperature range to only the interval $[\beta_{\text{min}}, 1]$, and enforcing these boundary conditions directly presents complications. We therefore follow the lead of \cite{continuous_tempering} by reparameterizing via
\begin{equation}
    \label{appendix:eq:reparam}
    \beta(\xi) \triangleq 
    \begin{cases} 
    1 & \text{if}\,\, \abs{\xi} < \Delta\\
    1-(1-\beta_{\text{min}})\left[3\left(\frac{\abs{\xi} - \Delta}{\Delta' - \Delta}\right)^2 - 2\left(\frac{\abs{\xi} - \Delta}{\Delta' - \Delta}\right)^3\right] & \text{if}\,\, \Delta \le \abs{\xi} \le \Delta'\\
    \beta_{\text{min}} & \text{if}\,\, \abs{\xi} > \Delta'   
    \end{cases}
\end{equation}
for all $\xi \in \Xi \triangleq \mathbb{R}$. A visualization of \cref{appendix:eq:reparam} is given in \cref{appendix:fig:reparam}.
 
 \begin{figure}[!t]
     \centering
     \includegraphics[width=\textwidth]{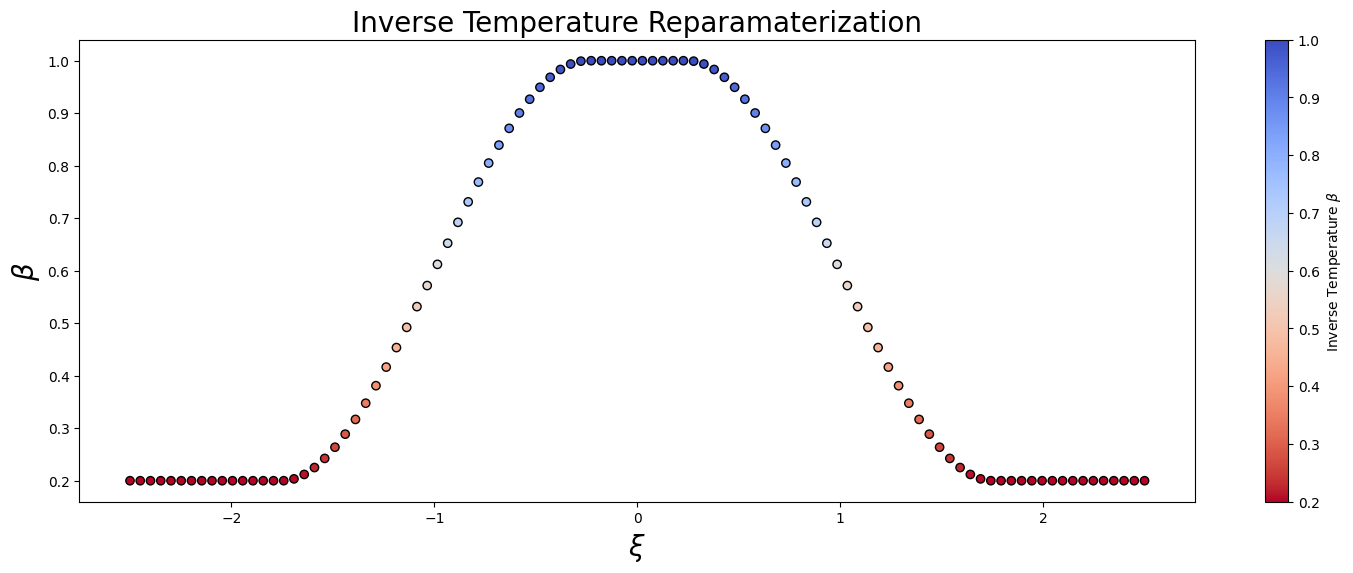}
     \caption{An illustration of the inverse temperature reparameterization given in \cref{appendix:eq:reparam}.}
     \label{appendix:fig:reparam}
 \end{figure}

\subsection{Biasing and Confining Potentials}
\label{appendix:biasing_and_confining_potentials}
Recall the joint density $\hat{\pi}_t(x,\xi)$ given in \cref{eq:xz_theta_joint_density} as
\begin{equation}
    \hat{\pi}_t^{\theta}(x, \xi) = e^{-U_t^{\xi}(x) + F_t^\theta(\xi) + \psi_t'(\xi)} = e^{-\tilde{U}_t^{\theta}(x,\xi)}.
\end{equation}
In practice, and to encourage the $\xi$-component to stay close to the critical interval $[-\Delta', \Delta']$ (see \cref{appendix:eq:reparam}), we may introduce choose $\psi_t'(\xi)$ so as to include a \textit{confining potential} $\psi^{\text{conf}}(\xi)$. In this paper, we take
\begin{equation}
    \label{appendix:eq:confining}
    \psi^{\text{conf}}(\xi) = 
    \begin{cases} 
    \eta(\xi + \tilde{\Delta})^2 & \text{if}\,\, \xi < - \tilde{\Delta}\\
    0 & \text{if}\,\, -\tilde{\Delta} \le \xi \le \tilde{\Delta}\\
    \eta(\xi - \tilde{\Delta})^2 & \text{if}\,\, \xi > \tilde{\Delta}
    \end{cases}
\end{equation}
where $\tilde{\Delta}$ (chosen to be slightly larger than $\Delta'$ from \cref{appendix:eq:reparam}) and $\eta > 0$ (denoting the sharpness of the confining potential) are hyperparameters. Additionally, one may consider alleviating a poorly initialized $F_t^{\theta}$ using an additional biasing force term $\psi_t^{\text{bias}}$ updated continuously throughout training to balance out the measured marginal distribution over $\xi$, so as to obtain the general form
\begin{equation}
    \psi_t'(\xi) = \psi^{\text{conf}}(\xi) + \psi_t^{\text{bias}}(\xi).
\end{equation}
In practice, we find that utilizing a biasing potential $\psi_t^{\text{bias}}(\xi)$ to be cumbersome and to provide little benefit over simpler solutions such as taking e.g., a larger value of $\sigma_{\xi}^2$ in \cref{eq:pi_dagger}.
\subsection{Characterizing Minimizers of the Multi-Temperature PINN Loss}
\label{appendix:mt_pinn_existence}
\begin{proof}[Proof of \cref{thm:ct_pinn_characterization}]
    We apply the basic argument of \citep{boltzmann_interpolation} in a $\xi$-pointwise fashion. Let us suppose that $(\mu^{\star}, F^{\star})$ satisfies $\mathcal{L}_{\pinn}(F^{\star},\mu^{\star}; \tilde{\pi}) = 0$. Then, for all $(x,\xi,t) \in \mathbb{R}^d \times \Xi \times \mathcal{T}$ (except perhaps on a set of measure zero), we must have 
    \begin{equation}
        \label{appendix:eq:pinn_satisfied}
        \partial_t F^{\star}_t(\xi) - \partial_t U_t^{\xi}(x) - \nabla_x \cdot \mu^{\star}_t(x,\xi) + (\nabla_x U_t^{\xi})^T \mu^{\star}_t(x,\xi) = 0.
    \end{equation}
    Now, observe that the true free energy $F_t(\xi)$ satisfies
    \begin{equation}
        \label{eq:free_energy_rearrange}
        \partial_t F_t(\xi) = - \partial_t \log Z_t(\xi) = \frac{-\int_{\mathbb{R}^d} \partial_t e^{-U_t^{\xi}(x)} \d x}{Z_t(\xi)} = \frac{\int_{\mathbb{R}^d} (\partial_t U_t^{\xi}(x))e^{-U_t^{\xi}(x)} \d x}{Z_t(\xi)} .
    \end{equation}
    However, by \cref{appendix:eq:pinn_satisfied}, we have that 
    \begin{equation}
        \label{appendix:eq:pinn_rearrange}
        \begin{aligned}
        \partial_t U_t^{\xi}(x,\xi) &= \partial_t F^{\star}_t(\xi) - \nabla_x \cdot \mu^{\star}_t(x,\xi) + (\nabla_x U_t^{\xi})^T \mu^{\star}_t(x,\xi)\\
        &= \partial_t F^{\star}_t(\xi) - e^{U_t(x)} \nabla_x \left[\mu_t^{\star}(x,\xi)e^{-U_t^{\xi}(x)}\right].
        \end{aligned}
    \end{equation}
    Plugging \cref{appendix:eq:pinn_rearrange} into \cref{eq:free_energy_rearrange} yields
    \begin{align*}
        \partial_t F_t(\xi) &= \frac{\int_{\mathbb{R}^d} \partial_t F_t^{\star}(\xi)e^{-U_t^{\xi}(x)}\d x}{Z_t(\xi)} - \frac{\int_{\mathbb{R}^d} \nabla_x \left[\mu_t^{\star}(x)e^{-U_t^{\xi}(x)}\right] \d x}{Z_t(\xi)}\\
        &= \frac{\int_{\mathbb{R}^d} \partial_t F_t^{\star}(\xi)e^{-U_t^{\xi}(x)}\d x}{Z_t(\xi)}\\
        &= \partial_t F_t^{\star}(\xi),
    \end{align*}
    where the second equality follows from the divergence theorem and the third from the definition of $Z_t(\xi)$. We conclude that $\partial_t F_t^{\star}(\xi) = \partial_t F_t(\xi)$ on $\mathcal{T}$, and together with the boundary condition $F_0^{\star}(\xi) = \partial_0 F_t(\xi)$, this is enough to prove the desired result.
\end{proof}

\section{Controlled Crooks and Jarzynski Results for Continuously Tempered Dynamics}
\label{appendix:crooks_and_jarzynski_section}
In this section, we establish a controlled Crooks fluctuation theorem for a controlled underdamped Langevin dynamics given by a \textit{non-separable Hamiltonian}, and as a corollary obtain an associated controlled Jarzynski equality. Our results partially generalize existing work for \textit{separable} Hamiltonians \citep{vaikuntanathan2008escorted, cmcd, zhong_asymmetric, nets}. 
\subsection{Background on Stochastic Integration}
\label{appendix:stochastic_background}

\begin{figure}[!h]
    \centering
    \renewcommand{\arraystretch}{1.2}
    \begin{tabular}{|c|>{\centering\arraybackslash}p{2cm}|>{\centering\arraybackslash}p{6cm}|}
        \multicolumn{3}{c}{} \\
        \hline
        Integral Name & Notation  & Discretization \\
        \hline 
        Forward Itô & $\int_0^T X_t \d Y_t$ & $\sum_i X_{t_i}(Y_{t_{i+1}} - Y_{t_i})$\\
        \hline
        Backward Itô & $\int_0^T X_t \bd Y_t$ & $\sum_i X_{t_{i+1}}(Y_{t_{i+1}} - Y_{t_i})$\\
        \hline 
        Stratonovich & $\int_0^T X_t \circ \d Y_t$ & $\sum_i \frac{1}{2}(X_{t_{i+1}} + X_{t_i})(Y_{t_{i+1}} - Y_{t_i})$\\
        \hline
    \end{tabular}
    \caption{Three flavors of stochastic integration.}
    \label{table:stochastic_integrals}
\end{figure}

In the remainder of this section, we'll make heavy use of forward and backward Itô integrals, as well as the Stratonovich integral. We largely follow the notation of \citep{cmcd}, and refer thereto for a more elaborate and technical discussion. For now, we summarize the three important Stochastic integrals - forward and backward Itô, and Stratonovich - in \cref{table:stochastic_integrals}, and provide some useful identities which relate this integrals, and which we shall make heavy use of. In particular, let us consider the prototypical forward SDE
\begin{equation}
    \label{eq:proto_forward}
    \d \fX_t = \mu_t(\fX_t) \d t + \sigma_t \d W_t,
\end{equation}
as shorthand for the forward Itô integral $\fX_t = \fX_0 + \int_0^t \mu_s(\fX_s)\d s + \int_0^t \sigma_s \d W_s$, and 
\begin{equation}
    \label{eq:proto_backward}
    \d \bX_t = \mu_t(\bX_t) \d t + \sigma_t \bd W_t,
\end{equation}
as shorthand for the backward Itô integral $\bX_t = \bX_T - \int_t^T \mu_s(\bX_s)\d s - \int_t^T \sigma_s \bd W_s$. We conclude with two useful identities which we make heavy use of, and which we provide intuitive, discretized intuition for. We again direct the reader to e.g., \citep{cmcd, kunita} for a more technical treatment. First, as is directly observed from the discretizations shown in \cref{table:stochastic_integrals}, we may relate the forward and backward Itô integrals to a corresponding Stratonovich integral via
\begin{equation}
    \label{eq:stratonovich_identity}
    \int_0^T X_t (\d Y_t + \bd Y_t) = 2 \int_0^T X_t \circ \d Y_t.
\end{equation}
Second, and less intuitively obvious, is the fact that for $X = \fX$ from \cref{eq:proto_forward},
\begin{equation}
    \label{eq:forward_backward_approx}
    \int_0^T \mu_t(X_t) (\bd X_t - \d X_t) = \lim_{\Delta t \to 0} \sum_i (\mu_{t + \Delta t}(X_{t + \Delta t}) - \mu_{t}(X_t))(X_{t + \Delta t} - X_t).
\end{equation}
Writing 
\begin{align}
    \mu_{t + \Delta t}(X_{t + \Delta t}) &= \mu_{t}(X_{t}) + \partial_t  \mu_{t}(X_{t}) + \nabla  \mu_{t}(X_{t})\d X_t\\
    &= \mu_{t}(X_{t}) + \partial_t  \mu_{t}(X_{t}) + \nabla  \mu_{t}(X_{t})\left[\mu_t(\fX_t) \d t + \sigma_t \d W_t\right],
\end{align}
plugging this into \cref{eq:forward_backward_approx}, and allowing the contribution of all $o(\Delta t)$ terms to vanish, we obtain  
\begin{equation}
    \label{eq:divergence_identity}
    \int_0^T \mu_t(X_t) (\bd X_t - \d X_t) = \int_0^T \sigma_t^2 \nabla \cdot \mu_t(X_t) \d t.
\end{equation}

\subsection{Radon-Nikodym Derivatives of Forward and Backward Langevin Dynamics}
\label{appendix:crooks_proofs}
Let us start by considering some time-dependent \textit{non-separable} Hamiltonian of the form 
\begin{equation}
    \label{eq:hamiltonian}
    \mathcal{H}_t(q_t,p_t) \triangleq U_t(q_t) + K(q_t, p_t).
\end{equation}
Then, we may define the controlled underdamped Langevin dynamics as the forward process  $\{\fX_t\}_{t \in [0,1]}$ given by
\begin{equation}
    \label{appendix:eq:forward_dynamics}
\begin{aligned}
    \d \fX_t = 
    \begin{bmatrix}
        \d q_t \\
        \d p_t 
    \end{bmatrix} = 
    \underbrace{
    \begin{bmatrix}
        \mu_t (q_t) \\
        0
    \end{bmatrix}\, \d t
    }_{\text{control}}
    + 
    \underbrace{
    \begin{bmatrix}
        \Gamma_t\nabla_p \mathcal{H}(q_t, p_t) \\
        - \Gamma_t\nabla_q \mathcal{H}(q_t, p_t)
    \end{bmatrix}\, \d t
    }_{\text{Hamiltonian dynamics}}
    +
    \underbrace{
    \begin{bmatrix}
        0 \\
        -\Gamma_tE_t\partial_p K(q_t, p_t)
    \end{bmatrix}\, \d t
    + 
    \begin{bmatrix}
        0 \\
        \sqrt{2 \Gamma_t E_T}
    \end{bmatrix} \d W_t
    }_{\text{Langevin dynamics}}.
    \end{aligned}
\end{equation}
where $X_0 \sim \pi_0 \propto e^{- \mathcal{H}_0}$ and where $\{W_t\}$ is a standard Brownian motion on $\mathbb{R}^d$. Here, we have defined the \emph{scaling} and \emph{damping coefficients}
\begin{equation}
    \Gamma_t \triangleq 
    \begin{bmatrix} 
    \gamma_t^1 & \hdots & 0 \\
    \vdots & \ddots & \vdots \\
    0 & \hdots & \gamma_t^d
    \end{bmatrix} \quad \quad \text{and} \quad \quad E_t \triangleq 
    \begin{bmatrix} 
    \varepsilon_t^1 & \hdots & 0 \\
    \vdots & \ddots & \vdots \\
    0 & \hdots & \varepsilon_t^d
    \end{bmatrix}.
\end{equation}

Then, for $T \in (0,1]$ may similarly consider the backward process $\{\bX_t\}_{t \in [0,1]}$.

\begin{equation}
    \label{appendix:eq:backward_dynamics}
\begin{aligned}
    \d \bX_t = 
    \begin{bmatrix}
        \d q_t \\
        \d p_t 
    \end{bmatrix} = 
    \begin{bmatrix}
        \mu_t (q_t) \\
        0
    \end{bmatrix}\, \d t
    +
    \begin{bmatrix}
        \Gamma_t\nabla_p \mathcal{H}(q_t, p_t) \\
        - \Gamma_t\nabla_q \mathcal{H}(q_t, p_t)
    \end{bmatrix}\, \d t
    -
    \begin{bmatrix}
        0 \\
        -\Gamma_tE_t\partial_p K(q_t, p_t)
    \end{bmatrix}\, \d t
    + 
    \begin{bmatrix}
        0 \\
        \sqrt{2 \Gamma_t E_T}
    \end{bmatrix} \bd W_t,
    \end{aligned}
\end{equation}
where $\bX_T \sim \pi_T \propto e^{- \mathcal{H}_T}$ and where $\bd$ denotes the backward Itô differential (see \cref{appendix:stochastic_background}). Let $\PP_T$ and $\QQ_T$ denote the path measures of $\{\fX_t\}_{t \in [0,T]}$ and $\{\bX_t\}_{t \in [0,T]}$. In what follows, we will denote by $Q_t$ and $P_t$ the $q$ and $p$ components respectively of the process $X_t$, and when necessary, by $\bQ_t$ and $\bP_t$ the respective $q$ and $p$ components of the backward process $\bX_t$. In this section we will establish the following result characterizing the Radon-Nikodym derivative (RND) of the forward and backward processes in \cref{appendix:eq:forward_dynamics} and \cref{appendix:eq:backward_dynamics}.

\begin{theorem}[RND of Forward and Backward Processes]
\label{appendix:thm:rnd}
The path measures $\PP_T$ and $\QQ_T$ are absolutely continuous with respect to one another, and for a sufficiently well-behaved proces $\{X_t\} = \{(Q_t, P_t\}$, the Radon-Nikodym derivative $\frac{\d \QQ_T}{\d \PP_T}(X)$ is given by 
\begin{align}
    \log \frac{\d \QQ_T}{\d \PP_T}(X) &=\log \pi_T(X_T) - \log \pi_0 (X_0) \\
    &+  \int_0^T \nabla_p K(q_t, p_t) \circ \d P_t + \int_0^T \nabla_q \cdot \mu_t(q_t) + (\Gamma_t \nabla_q \mathcal{H}_t(q_t,p_t))^T\nabla_p K(q_t,p_t) \d t.
    \label{appendix:eq:rnd}
\end{align}
\end{theorem}
One challenging aspect of \cref{appendix:eq:forward_dynamics} (and, in turn, of \cref{appendix:eq:backward_dynamics}) is the fact that we are only injecting noise into the momentum component, so that the driving Brownian motion $W_t$ has dimension only half that of the phase space on which these two processes are defined. Accordingly, we recall the following Girsanov result which accommodates such diffusion processes involving a lower dimensional driving Brownian motion.
\begin{theorem}[Multivariate Extension of Theorem 7.19, \cite{liptser2013statistics}]
\label{appendix:thm:external_girsanov}
Let $T \in [0,1]$ and consider the two diffusion processes $\{Y\}_{t \in [0,T]}, \{Z\}_{t \in [0,T]}$, on $\RR^n$ given by 
\begin{align*}
    \d Y_t &= \phi_t(Y_t)\d t + \Sigma_t \d W_t \\
    \d Z_t &= \psi_t(Z_t)\d t + \Sigma_t \d W_t 
\end{align*}
with $Y_0 = Z_0$ fixed, so that $\mu, \phi, \Sigma$ satisfy mild regularity conditions, where $\{W_t\}$ is a standard Brownian motion on $\RR^k$, and where $\Sigma_t \in \mathbb{R}^{n \times k}$. Additionally, let $\QQ_T^Y$ and $\QQ_T^Z$ denote the respective path measures of $Y$ and $Z$. Then, if 
\begin{equation}
    \int_0^T \left[\phi_t(X_t)^T(\Sigma_t\Sigma_t^T)^{+} \phi_t(X_t) + \psi_t(X_t)^T(\Sigma_t\Sigma_t^T)^{+} \psi_t(X_t)\right]\,\d t < \infty
\end{equation}
both $\QQ_T^Y$ and $\QQ_T^Z$ almost surely, then $\QQ_T^Y \sim \QQ_T^Z$ are absolutely continuous with respect to one another and their Radon-Nikodym derivative is given by
\begin{equation}
    \label{appendix:eq:external_girsanov}
    \frac{\d \QQ_T^Y}{\d \QQ_T^Z}(Z) = 
    \exp \left(
    \int_0^T(\phi_t(z) - \psi_t(z))^T(\Sigma_t\Sigma_t^T)^{+}(\d Z_t - \frac{1}{2}(\phi_t(z) + \psi_t(z)) \d t)\right),
\end{equation}
where $A^+$ denotes the Moore-Penrose pseudoinvere of a matrix $A$.
\end{theorem}

We may now proceed to the proof of Theorem \ref{appendix:thm:rnd}.
\begin{proof}[Proof of \cref{appendix:thm:rnd}]
First, let us recall the forward process from \cref{appendix:eq:forward_dynamics}, given as
\begin{equation}
\label{appendix:eq:forward_compact}
\begin{aligned}
    \d \fX_t = 
    \begin{bmatrix}
        \d q_t \\
        \d p_t 
    \end{bmatrix} = 
    \begin{bmatrix}
        \mu_t (q_t) \\
        0
    \end{bmatrix}\, \d t
    +
    \begin{bmatrix}
        \Gamma_t\nabla_p \mathcal{H}(q_t, p_t) \\
        - \Gamma_t\nabla_q \mathcal{H}(q_t, p_t)
    \end{bmatrix}\, \d t
    +
    \begin{bmatrix}
        0 \\
        -\Gamma_tE_t\partial_p K(q_t, p_t)
    \end{bmatrix}\, \d t
    + 
    \begin{bmatrix}
        0 \\
        \sqrt{2 \Gamma_t E_t}
    \end{bmatrix} \d W_t,
\end{aligned}
\end{equation}
and a corresponding backward process given by
\begin{equation}
\label{appendix:eq:backward_compact}
\begin{aligned}
    \d \bX_t = 
    \begin{bmatrix}
        \mu_t (q_t) \\
        0
    \end{bmatrix}\, \d t
    &+
    \begin{bmatrix}
        \Gamma_t\nabla_p \mathcal{H}(q_t, p_t) \\
        - \Gamma_t\nabla_q \mathcal{H}(q_t, p_t)
    \end{bmatrix}\, \d t
    -
    \begin{bmatrix}
        0 \\
        -\Gamma_tE_t\partial_p K(q_t, p_t)
    \end{bmatrix}\, \d t
    + 
    \begin{bmatrix}
        0 \\
        \sqrt{2 \Gamma_t E_T}
    \end{bmatrix} \bd W_t.
    \end{aligned}
\end{equation}
whose path measures on $[0,T]$ we have denoted by $\PP_T$ and $\QQ_T$ respectively. At a high level, we will employ the same trick as was utilized concurrently in both \cite{cmcd} and \cite{improved_sampling} of introducing forward and backward references processes with identical path measures, so as to obtain an expression, via cancellation, of the desired RND between the original forward and backward processes. However, the approach must be adapted to handle the singular diffusion coefficient $\Sigma_t$. Explicitly, for some choice of drift $\nu': \mathcal{T} \times \RR^{2d} \to \RR^d$, let us first consider a forward reference process $\{\fX_t^{\nu', \reff}\}$ of the form 
\begin{equation}
    \d \fX_t^{\nu', \reff} = 
    \begin{bmatrix}
        \d q_t \\
        \d p_t 
    \end{bmatrix} = 
    \begin{bmatrix}
        \mu_t (q_t) + \Gamma_t\nabla_p \mathcal{H}(q_t, p_t)\\
        \nu_t'(q_t,p_t)
    \end{bmatrix}\, \d t
    + 
    \begin{bmatrix}
        0 \\
        \sqrt{2\Gamma_tE_t}
    \end{bmatrix} \d W_t, \quad \quad X_0^{\phi, \reff} \sim \Lambda
    \label{appendix:eq:nu_forward_ref_dynamics}
\end{equation}
where $\Lambda$ denotes the Lebesgue measure on phase space $\RR^{2d}$. Let us now choose $\nu$ so that $X_t^{\nu', \reff} \sim \Lambda$ is distributed like the Lebesgue measure for all $t \in [0,1]$. To do so, it is sufficient to take $\nu'$ so that the drift in \cref{appendix:eq:nu_forward_ref_dynamics} is divergence free, or equivalently, that 
\begin{equation}
    \label{eq:nu_property}
    \nabla_p \cdot \nu_t'(q_t,p_t) = - \nabla_q \cdot (\mu_t(q_t) + \Gamma_t\nabla_p \mathcal{H}_t(q_t, p_t))
\end{equation}
This is easily achieved by setting 
\begin{equation}
    \label{eq:nu}
    \nu'_t(q_t,p_t) = \nu_t(q_t,p_t) = - \diag(p_t)\diag(\partial_q \mu_t(q_t)) - \Gamma_t \nabla_q \mathcal{H}_t(q_t,p_t)
\end{equation}
where $\partial_q \mu_t(q_t)$ denotes the Jacobian of $\mu_t$, and where the diagonal operator $\diag$ acts on vectors and (square) matrices as 
\begin{equation}
    \diag(p_t) \triangleq 
    \begin{bmatrix}
        p_1 & \hdots & 0 \\
        \vdots & \vdots & 0 \\
        0 & \hdots & p_d
    \end{bmatrix} \in \mathbb{R}^{d \times d}, 
    \quad \text{and} \quad \diag(\partial_q \mu_t(q_t)) \triangleq
    \begin{bmatrix}
        (\partial_q \mu_t(q_t))_{11} \\
        \vdots \\
        (\partial_q \mu_t(q_t))_{dd} \\
    \end{bmatrix} \in \mathbb{R}^d,
\end{equation}
respectively embedding a vector as a diagonal, or extracting from a matrix its diagonal. Plugging $\nu' = \nu$ from \cref{eq:nu} into \cref{appendix:eq:nu_forward_ref_dynamics}, we obtain 
\begin{equation}
    \d \fX_t^{\reff} \triangleq \d \fX_t^{\nu, \reff} = 
    \begin{bmatrix}
        \mu_t (q_t) + \Gamma_t\nabla_p \mathcal{H}_t(q_t,p_t)\\
         \nu(q_t,p_t)
    \end{bmatrix}\, \d t
    + 
    \begin{bmatrix}
        0 \\
        \sqrt{2\Gamma_tE_t}
    \end{bmatrix} \d W_t, \quad \quad X_0^{\reff} \sim \Lambda
    \label{appendix:eq:forward_ref_dynamics}
\end{equation}
Denoting by $\PP^{\reff}$ the path measure of the process $\fX^{\reff}$, we may verify that the regularity conditions of \cref{appendix:thm:external_girsanov} hold, so that we may then apply the result to $Y = \fX$ and $Z = \fX^{\reff}$ with
\begin{align*}
    \phi_t \triangleq \begin{bmatrix}
        \mu_t(q_t) + \Gamma_t \nabla_p \mathcal{H}_t(q_t,p_t) \\
        - \Gamma_t \nabla_q \mathcal{H}_t(q_t,p_t) - \Gamma_tE_t \nabla_p K(q_t, p_t)
    \end{bmatrix}, \quad \psi_t \triangleq \begin{bmatrix}
        \mu_t(q_t) + \Gamma_t \nabla_p \mathcal{H}_t(q_t,p_t)  \\
        \nu_t(q_t,p_t)
    \end{bmatrix}, \quad \Sigma_t \triangleq \begin{bmatrix}
        0 \\ \sqrt{2\Gamma_t E_t}
    \end{bmatrix}.
\end{align*}
It follows that $\PP_T \sim \PP_T^{\text{ref}}$ are absolutely continuous with respect to one another. Observe now that $\Sigma_t \Sigma_t^T = \begin{bmatrix} 0 & 0 \\ 0 & 2\Gamma_t E_t \end{bmatrix}$, so that its pseudo-inverse is given by
\begin{equation*}
    (\Sigma_t \Sigma_t^T)^{+}= \begin{bmatrix} 0 & 0 \\ 0 & \frac{1}{2}(\Gamma_TE_t)^{-1} \end{bmatrix}.
\end{equation*}
Writing $A_t = A_t(q_t,p_t) \triangleq \Gamma_t \nabla_q \mathcal{H}_t(q_t,p_t) + \Gamma_tE_t \nabla_p \mathcal{H}_t(q_t,p_t)$, \cref{appendix:eq:external_girsanov} then immediately yields
\begin{align}
     \log \frac{\d \PP}{\d \PP^{\reff}}(X) &= \log \left(\frac{\d \pi_0}{\d E_0}\right)(X_0) + \int_0^T \frac{1}{2}(-A_t - \nu_t)^T(\Gamma_tE_t)^{-1}(\d P_t - \frac{1}{2}(-A_t + \nu_t)\d t)
    \label{appendix:eq:forward_rnd}
\end{align}
Let us now turn our attention towards the backward process from \cref{appendix:eq:backward_dynamics}, which we recall below as
\begin{equation}
    \d \bX_t = 
    \begin{bmatrix}
        \mu_t (q_t) \\
        0
    \end{bmatrix} \d t
    +
    \begin{bmatrix}
        \Gamma_t\nabla_p \mathcal{H}_t(q_t,p_t) \\
        - \Gamma_t\nabla_q \mathcal{H}_t(q_t,p_t)
    \end{bmatrix} \d t
    -
    \begin{bmatrix}
        0 \\
        -\Gamma_tE_t\partial_p K(q_t, p_t)
    \end{bmatrix} \d t
    + 
    \begin{bmatrix}
        0 \\
        \sqrt{2 \Gamma_t E_t}
    \end{bmatrix} \bd W_t.
\end{equation}
from which we may construct the reference process
\begin{equation}
    \d \bX_t^{\reff} = 
    \begin{bmatrix}
        \mu_t (q_t) + \Gamma_t\nabla_p K(p_t)\\
         \nu_t(q_t, p_t)
    \end{bmatrix}\, \d t
    + 
    \begin{bmatrix}
        0 \\
        \sqrt{2\Gamma_tE_t}
    \end{bmatrix} \bd W_t, \quad \quad X_0^{\reff} \sim \Lambda
    \label{appendix:eq:backward_ref_dynamics}
\end{equation}
whose path measure we shall denote by $\QQ_T^{\text{ref}}$ and so that $\bX_t^{\text{ref}} \sim \Lambda$ is distributed like the Lebesgue measure for all $t \in [0,T]$. By introducing $s \triangleq T - t$, we may instead consider the forward (in $s$) process $\{Y_s\} \triangleq \{\bX_{1-s}\}$ and $\{Y^{\reff}_s\} \triangleq \{\bX_{1-s}^{\reff}\}$ given by
\begin{equation}
    \d Y_s = 
   -\begin{bmatrix}
        \tilde{\mu}_s (q_s) \\
        0
    \end{bmatrix} \d s
    -
    \begin{bmatrix}
        \tilde{\Gamma}_s\nabla_p \mathcal{H}_s(q_s,p_s) \\
        - \tilde{\Gamma}_s\nabla_q \mathcal{H}_s(q_s,p_s)
    \end{bmatrix} \d s
    +
    \begin{bmatrix}
        0 \\
        -\tilde{\Gamma}_s\tilde{E}_s\nabla_p K(q_s, p_s)
    \end{bmatrix}\d s
    + 
    \begin{bmatrix}
        0 \\
        \sqrt{2 \tilde{\Gamma}_s \tilde{E}_s}
    \end{bmatrix} \d W_s
    , \quad \quad Y_0 \sim \pi_T \propto e^{- \mathcal{H}_T},
    \label{appendix:eq:backward_dynamics_s}
\end{equation}
and the associated reverse-time reference process
\begin{equation}
    \d Y_s^{\reff} = 
    \begin{bmatrix}
        - \tilde{\mu}_s (q_s) - \tilde{\Gamma}_s\nabla_p \mathcal{H}_s(q_s,p_s)\\
        -\tilde{\nu}(q_s, p_s)
    \end{bmatrix}\d s
    + 
    \begin{bmatrix}
        0 \\
        \sqrt{2 \tilde{\Gamma}_s \tilde{E}_s}
    \end{bmatrix} \d W_s, \quad \quad Y_0^{\reff} \sim \Lambda
    \label{appendix:eq:backward_ref_dynamics_s}
\end{equation}
 where we have defined e.g., $\tilde{\mu}_s = \mu_{T-s}$, and where $Y_s^{\text{ref}} \sim \Lambda$ is inherited from \cref{appendix:eq:backward_ref_dynamics}. Note that both equations now involve the forward Itô differential $\d$ rather than the backward differential $\bd$. Denoting by $\tilde{\QQ}_T$ and $\tilde{\QQ}_T^{\reff}$ the path measures of the time-reversed processes $\{Y_s\}$ and $\{Y^{\reff}_s\}$, we may again verify the necessary regularity conditions and apply \cref{appendix:thm:external_girsanov} with $Y = \{Y_s\}$ and $Z = \{Y_s^{\text{ref}}\}$ so that we have
\begin{align}
    \phi_s \triangleq \begin{bmatrix}
        -\tilde{\mu}_s(q_s) - \tilde{\Gamma}_s \nabla_p \tilde{\mathcal{H}}_s(q_s,p_s)  \\
         \tilde{\Gamma}_s \nabla_q \tilde{\mathcal{H}}_s(q_s,p_s) - \tilde{\Gamma}_s\tilde{E}_s\nabla_p K(q_s, p_s)
    \end{bmatrix}, \quad \psi_s \triangleq \begin{bmatrix}
        -\tilde{\mu}_s(q_s) - \tilde{\Gamma}_s \nabla_p \tilde{\mathcal{H}}_s(q_s,p_s) \\
        - \tilde{\nu}_s(q_s, p_s)
    \end{bmatrix}, \quad \Sigma_s \triangleq \begin{bmatrix}
        0 \\ \sqrt{2 \tilde{\Gamma}_s \tilde{E}_s}
    \end{bmatrix}.
\end{align}
yielding 
\begin{equation}
\label{eq:s_rnd}
\begin{aligned}
    \log \frac{\d \tilde{\QQ}_T}{\d \tilde{\QQ}_T^{\reff}}(Y)
    = \log \left(\frac{\d \pi_T}{\d \Lambda}\right)(X_0) + \int_0^T \frac{1}{2}(-\tilde{C}_s + \tilde{\nu}_s)^T(\tilde{\Gamma}_s\tilde{E}_s)^{-1}(\d P_s - \frac{1}{2}(-\tilde{C}_s - \tilde{\nu}_s) \d s).
\end{aligned}
\end{equation}
 where we have introduced $\tilde{C}_s(q_s,p_s) = C_{T-s}(q_s,p_s)$ with $C_t(q_t, p_t) \triangleq A_t(q_t,p_t) - 2\Gamma_t\nabla_q \mathcal{H}_t(q_t, p_t)$. Denoting by $\QQ^{\reff}$ the path measure of $\{\bX^{\reff}_t\}$ , it is easily shown that
\begin{equation}
    \label{eq:forward_backward_rnds}
    \log \frac{\d \tilde{\QQ}_T}{\d \tilde{\QQ}_T^{\reff}}(Y) = \log \frac{\d \QQ_T}{\d \QQ_T^{\reff}}(\bX),
\end{equation}
where $Y_s = \bX_{T-s}$. It therefore follows from \cref{eq:forward_backward_rnds}, the relation $\tilde{f}_s \d \tilde{P}_s = -f_t \bd P_t$, and \cref{eq:s_rnd} that
\begin{align}
    \log \frac{\d \QQ_T}{\d \QQ_T^{\reff}}(X) = \log \left(\frac{\d \pi_T}{\d \Lambda}\right)(X_T)+ \int_0^T  \frac{1}{2} (-C_t +\nu_t)^T(\Gamma_tE_t)^{-1}(-\bd P_t - \frac{1}{2}(-C_t - \nu_t) \d t).
     \label{appendix:eq:backward_rnd}
\end{align}
To finish, we may exploit the fact that since $\QQ_T^{\reff}= \PP_T^{\reff}$ are the same to find then
\begin{equation}
    \label{appendix:eq:rnd_relation}
    \log \frac{\d \QQ_T}{\d \PP_T}(X) = \log \frac{\d \QQ_T}{\d \QQ_T^{\reff}}(X) -  \log \frac{\d \PP_T}{\d \PP_T^{\reff}}(X) - \log \frac{\d \PP_T^{\text{ref}}}{\d \QQ_T^{\text{ref}}}(X).
\end{equation}
To see that the last term is zero, it suffices to note that $\bX_t^{\text{ref}} \sim \Lambda$ is given by the Lebesgue measure for all $t \in [0,T]$, and since the Lebesgue measure has vanishing score, the forward reference process \cref{appendix:eq:forward_ref_dynamics} is exactly the time-reversal of the backward reference process \cref{appendix:eq:backward_ref_dynamics}, and thus $\log \frac{\d \PP_T^{\text{ref}}}{\d \QQ_T^{\text{ref}}}(X) = 0$. We may therefore plug in \cref{appendix:eq:forward_rnd} and \cref{appendix:eq:backward_rnd} into \cref{appendix:eq:rnd_relation} to obtain

\begin{align}
    \log \frac{\d \QQ_T}{\d \PP_T}(X) &= \log \left(\frac{\d \pi_T}{\d \Lambda}\right)(X_T) - \log \left(\frac{\d \pi_0}{\d \Lambda}\right)(X_0)\\
    &+\int_0^T  \frac{1}{2} (-C_t + \nu_t)^T(\Gamma_tE_t)^{-1}(\textcolor{BrickRed}{-\bd P_t} \textcolor{RoyalBlue}{- \frac{1}{2}(-C_t - \nu_t) \d t}) \\
    &- \int_0^T \frac{1}{2}(-A_t - \nu_t)^T(\Gamma_tE_t)^{-1}(\textcolor{BrickRed}{\d P_t}\textcolor{RoyalBlue}{- \frac{1}{2}(-A_t + \nu_t)\d t}).
\end{align}
We may now simplify the portions of the integrand correspond to each of the \textcolor{BrickRed}{red} and \textcolor{MidnightBlue}{blue} terms. 

{\bfseries \textcolor{BrickRed}{Red Terms:}} The \textcolor{BrickRed}{red} terms combine to
\begin{align}
    \label{appendix:eq:red_simple}
    \int_0^T \frac{1}{2} \left[(C_t^T(\Gamma_tE_t)^{-1}\bd P_t + A_t^T(\Gamma_tE_t)^{-1}\d P_t\right] + \int_0^T \frac{1}{2} \nu_t^T(\Gamma_tE_t)^{-1} (\d P_t - \bd P_t).
\end{align}
Writing $D_t \triangleq \Gamma_t \nabla_q \mathcal{H}_t(q_t,p_t)$, so that $A_t = C_t + 2D_t$, the first term simplifies as
\begin{equation}
\label{eq:red_first_term}
\begin{aligned}
    &\int_0^T \frac{1}{2} \left[(C_t^T(\Gamma_tE_t)^{-1}\bd P_t + A_t^T(\Gamma_tE_t)^{-1}\d P_t\right]\\
    &= \frac{1}{2}\int_0^T (A_t - D_t)^T(\Gamma_tE_t)^{-1}(\d P_t + \bd P_t) + \frac{1}{2} \int_0^T D_t^T(\Gamma_tE_t)^{-1}(dP_t - \bd P_t)\\
    &= \frac{1}{2}\int_0^T (\nabla_p K(q_t,p_t))^T(\d P_t + \bd P_t) - \int_0^T \nabla_p \cdot D_t \d t\\
    &= \int_0^T \nabla_p K(q_t, p_t)^T \circ \d P_t - \int_0^T \nabla_p \cdot D_t \d t,
\end{aligned}
\end{equation}
where in the first and second equalities we have utilized the definitions of $A_t$, $C_t$, and $D_t$, in the second we have used \cref{eq:divergence_identity}, and in the third \cref{eq:stratonovich_identity}. The second term simplifies as 
\begin{equation}
    \label{eq:red_second_term}
    \begin{aligned}
    \int_0^T \frac{1}{2} \nu_t^T(\Gamma_tE_t)^{-1} (\d P_t - \bd P_t) &= - \int_0^T \nabla_p \cdot \nu_t \d t\\
    &= \int_0^T \nabla_q \cdot \mu_t(q_t) + \nabla_q \cdot (\Gamma_t \nabla_p \mathcal{H}_t(q_t,p_t)) \d t\\
    &= \int_0^T \nabla_q \cdot \mu_t(q_t) + \nabla_p \cdot D_t \d t
    \end{aligned}
\end{equation}
where we have utilized \cref{eq:divergence_identity} and \cref{eq:nu_property}. Plugging \cref{eq:red_first_term} and \cref{eq:red_second_term} into \cref{appendix:eq:red_simple} yields
\begin{equation}
    \label{eq:red_simpler}
    \tag{\textcolor{BrickRed}{$\bigstar$}}
    \int_0^T \nabla_p K(q_t, p_t)^T \circ \d P_t + \int_0^T \nabla_q \cdot \mu_t(q_t)\d t.
\end{equation}

{\bfseries \textcolor{MidnightBlue}{Blue Terms:}} The \textcolor{MidnightBlue}{blue} terms combine to
\begin{align}
    -\frac{1}{4}\int_0^1  \left[(-C_t + \nu_t)^T(\Gamma_tE_t)^{-1}(-C_t - \nu_t) - (-A_t - \nu_t)^T(\Gamma_tE_t)^{-1}(-A_t + \nu_t)\right] \d t
\end{align}
which readily simplifies to
\begin{equation}
    \label{eq:blue_simple}
    \tag{\textcolor{MidnightBlue}{$\bigstar$}}
\begin{aligned}
    &-\frac{1}{4} \int_0^T C_t^T (\Gamma_tE_t)^{-1} C_t - A_t (\Gamma_tE_t)^{-1} A_t \d t\\
    &= -\frac{1}{4} \int_0^T (A_t - 2D_t)^T (\Gamma_tE_t)^{-1} (A_t - 2D_t) - A_t (\Gamma_tE_t)^{-1} A_t \d t\\
    &= \int_0^T D_t^T (\Gamma_tE_t)^{-1} (A_t - D_t) \d t\\
    &= \int_0^T (\Gamma_t\nabla_q \mathcal{H}_t(q_t,p_t))^T (\Gamma_tE_t)^{-1} (E_t \Gamma_t \nabla_p \mathcal{H}(q_t, p_t))\d t\\
    &= \int_0^T  (\Gamma_t\nabla_q \mathcal{H}_t(q_t,p_t))^T \nabla_p K(q_t,p_t)\d t
\end{aligned}
\end{equation}
Combining \cref{eq:blue_simple} with \cref{eq:red_simpler}, we obtain
\begin{align}
    \log \frac{\d \QQ_T}{\d \PP_T}(X) &= \log \left(\frac{\d \pi_T}{\d \Lambda}\right)(X_T) - \log \left(\frac{\d \pi_0}{\d \Lambda}\right)(X_0)\\
    &+  \int_0^T \nabla_p K(p_t) \circ \d P_t + \int_0^T \nabla_q \cdot \mu_t(q_t) + (\Gamma_t\nabla_q \mathcal{H}_t(q_t,p_t))^T \nabla_p K(q_t,p_t) \d t.
\end{align}
The desired result then follows from the observation that $
    \log \left(\frac{\d \pi_T}{\d \Lambda}\right)(X_T) - \log \left(\frac{\d \pi_0}{\d \Lambda}\right)(X_0) = \log \pi_T(X_T) - \log \pi_0 (X_0)$.
\end{proof}

\subsection{Controlled Crooks and Jarzynski Equalities for Non-Separable Langevin Dynamics}

We now state and prove a controlled generalization of the Crooks fluctuation theorem (see \citep{Crooks_1999}) relating the dynamics from \cref{appendix:eq:forward_dynamics} and \cref{appendix:eq:backward_dynamics} for the non-separable Hamiltonian from \cref{eq:hamiltonian}.
\begin{theorem}[Controlled Crooks Fluctuation Theorem for Non-Separable Langevin Dynamics]
    \label{appendix:thm:crooks}
    For any $T \in (0,1]$, and with $\PP_T$ and $\QQ_T$ defined as as the path measures of \cref{appendix:eq:forward_dynamics} and \cref{appendix:eq:backward_dynamics} on the interval $[0,T]$, then for any $X \sim \PP_T$, we have
    \begin{equation}
        \frac{\d \QQ_T}{\d \PP_T}(\fX) = \exp\left(F_T - F_0 + \int_0^T \nabla_q \cdot \mu_t(q_t) - \partial_t \log U_t(q_t) - \mu_t(q_t)^T \nabla_q \mathcal{H}_t(q_t,p_t) \d t\right).
        \label{appendix:eq:crooks}
    \end{equation}
\end{theorem}

\begin{proof}
By Theorem \ref{appendix:thm:rnd},
\begin{align}
    \log \frac{\d \QQ_T}{\d \PP_T}(X) &= \log \pi_T(X_T) - \log \pi_0 (X_0)\\
    &+ \int_0^T \nabla_p K(q_t, p_t) \circ \d P_t + \int_0^T \nabla_q \cdot \mu_t(q_t) + (\Gamma_t\nabla_q \mathcal{H}_t(q_t,p_t))^T \nabla_p K(q_t,p_t) \d t.
    \label{appendix:eq:rnd2}
\end{align}
Observe now that via the Stratonovich formulation of Itô's lemma applied to $\log \hat{\pi}_t = \log \pi_t + F_t$, 
\begin{align}
    \log \hat{\pi}_1 (X_1) - \log \hat{\pi}_0(X_0) &= \int_0^T \partial_t \log \hat{\pi}_t (X_t) \d t + \int_0^T \partial_x \log \hat{\pi}_t (X_t)\circ \d X_t\\
    &= -\int_0^T \partial_t \mathcal{H}_t(q_t,p_t) \d t - \int_0^T (\nabla_q \mathcal{H}_t(q_t,p_t))^T \circ \d Q_t - \int_0^T (\nabla_p \mathcal{H}_t(q_t,p_t))^T \circ \d P_t\\
    &= -\int_0^T \partial_t U_t(q_t) \d t - \int_0^T (\nabla_q \mathcal{H}_t(q_t,p_t))^T \circ \d Q_t - \int_0^T (\nabla_p K(q_t,p_t))^T \circ \d P_t.
    \label{appendix:eq:stratonovich_ito}
\end{align}
Since $X \sim \PP_T$ is given by \cref{appendix:eq:forward_dynamics}, we may plug in 
\begin{equation}
   \d Q_t = \left[\mu_t(q_t) + \Gamma_t \nabla_p \mathcal{H}_t(q_t,p_t)\right]\d t =\left[\mu_t(q_t) + \Gamma_t\nabla_p K(q_t, p_t)\right] \d t
\end{equation}
to obtain
\begin{equation}
    \int_0^T (\Gamma_t \nabla_q \mathcal{H}(q_t, p_t))^T\nabla_p K(q_t,p_t) \d t = \int_0^T \nabla_q \mathcal{H}_t(q_t,p_T)^T \circ \d Q_t - \int_0^T \mu_t(q_t)^T \nabla_q \mathcal{H}_t(q_t,p_t) \d t.
\end{equation}
Equation \ref{appendix:eq:stratonovich_ito} thus implies that 
\begin{align}
    \int_0^T (\Gamma_t \nabla_q \mathcal{H}(q_t, p_t))^T\nabla_p K(q_t,p_t) \d t &= \log \hat{\pi}_0 (X_0) - \log \hat{\pi}_1(X_1)- \int_0^T \partial_t \log U_t(q_t) \d t\\
    &- \int_0^T (\nabla_p \mathcal{H}_t(q_t,p_t))^T \circ \d P_t - \int_0^T  \mu_t(q_t)^T \nabla_q \mathcal{H}_t(q_t,p_t) \d t.
    \label{appendix:eq:crooks_intermediate}
\end{align}
Plugging Equation \ref{appendix:eq:crooks_intermediate} back into Equation \ref{appendix:eq:rnd2} we have
\begin{align*}
        \log \frac{\d \QQ_T}{\d \PP_T}(X) &= \log \pi_T(X_T) - \log \pi_0 (X_0) + \log \hat{\pi}_0 (X_0) - \log \hat{\pi}_T(X_T) \\
    &+  \int_0^T \nabla_p K(q_t, p_t) \circ \d P_t + \int_0^T \nabla_q \cdot \mu_t(q_t)\d t\\
    &- \int_0^T \partial_t \log U_t(q_t) \d t - \int_0^T \nabla_p K(q_t, p_t) \circ \d P_t - \int_0^T \mu_t(q_t)^T \nabla_q \mathcal{H}_t(q_t,p_t) \d t.
\end{align*}

Simplifying then yields
\begin{align}
    \log \frac{\d \QQ_T}{\d \PP_T}(X) &=  F_T - F_0 + \int_0^T \nabla_q \cdot \mu_t(q_t) - \partial_t \log U_t(q_t) - \mu_t(q_t)^T \nabla_q \mathcal{H}_t(q_t,p_t) \d t
\end{align}
from which the desired result immediately follows.
\end{proof}

We finish by noting that the Radon-Nikodym derivative of Theorem \ref{appendix:thm:crooks} allows us to importance sample from $\QQ_T$ using $\PP_T$, as is formalized in the following controlled Jarzynski equality.
\begin{corollary}[Controlled Jarzynski Equality]
For $T \in (0,1]$, let $h \in \mathcal{C}^1([0,T],\mathbb{R})$ denote some observable. Then 
\begin{equation}
    \EE_{X \sim \QQ_T} \left[h(X)\right] = \frac{\EE_{X \sim \PP_T} \left[ h(X)\exp\left(A_t(X)\right)\right]}{\EE_{X \sim \PP_T} \left[ \exp\left(A_t(X)\right)\right]}.
\end{equation}
where the work functional $A_T(X)$ is given by 
\begin{equation}
    \label{appendix:eq:work}
    A_T(X) \triangleq \int_0^T \nabla_q \cdot \mu_t(q_t) - \partial_t \log U_t(q_t) - \mu_t(q_t)^T \nabla_q \mathcal{H}_t(q_t, p_t) \d t.
\end{equation}
In particular, 
\begin{equation}
    \EE_{X \sim \PP_T} \left[\exp(A_T(X)) \right] = \exp(F_0 - F_T) = \frac{Z_T}{Z_0}.
    \label{appendix:eq:jarzynski}
\end{equation}
\label{appendix:thm:jarzynski}
\end{corollary}
A similar result is established in \citep{nets}, for an overdamped Langevin dynamics, and in the special case that $h(X) = h(X_T)$ is a function of the terminal point.
\begin{proof}
Observe that by \cref{appendix:thm:crooks},
\begin{equation*}
    \EE_{X \sim \PP_T} \left[\exp(A_T(X)) \right] = \exp(F_0 - F_T)\EE_{X \sim \PP_T} \left[\frac{\d \QQ_T}{\d \PP_T}(X)\right] = \exp(F_0 - F_T).
\end{equation*}
Thus,
\begin{align*}
    \frac{\EE_{X \sim \PP_T} \left[ h(X)\exp\left(A_T(X)\right)\right]}{\EE_{X \sim \PP_T} \left[ \exp\left(A_T(X)\right)\right]} &= \exp(F_0 - F_T)\EE_{X \sim \PP_T} \left[ h(X)\exp\left(A_t(X)\right)\right]\\
    &= \EE_{X \sim \PP_T} \left[ h(X) \frac{\d \QQ_T}{\d \PP_T}(X)\right]\\
    &= \EE_{X \sim \QQ_T} \left[ h(X) \right],
\end{align*}
as desired.
\end{proof}

\subsection{Specialization to the Continuously Tempered Setting}
\label{appendix:ct_specialization}
\begin{proof}[Proof of \cref{thm:ct_jarzynski}]
    Recall the CTDS dynamics from \cref{eq:ctds}, which may be rewritten in the form of \cref{appendix:eq:forward_dynamics} as 
\begin{equation}
\begin{aligned}
    \d X_t^{\text{CTDS}} = 
    \begin{bmatrix}
        \d q_t \\
        \d p_t 
    \end{bmatrix} = 
    \underbrace{
    \begin{bmatrix}
        \mu_t^\theta(q_t) \\
        0
    \end{bmatrix}\, \d t
    }_{\text{control}}
    + 
    \underbrace{
    \begin{bmatrix}
        \Gamma_t\nabla_p \mathcal{H}_t(q_t, p_t) \\
        - \Gamma_t\nabla_q \mathcal{H}_t(q_t, p_t)
    \end{bmatrix}\, \d t
    }_{\text{Hamiltonian dynamics}}
    +
    \underbrace{
    \begin{bmatrix}
        0 \\
        -\Gamma_tE_t\partial_p K(q_t, p_t)
    \end{bmatrix}\, \d t
    + 
    \begin{bmatrix}
        0 \\
        \sqrt{2 \Gamma_t E_T}
    \end{bmatrix} \d W_t
    }_{\text{Langevin dynamics}}.
    \end{aligned}
\end{equation}
where $q_t = (X_t, \xi_t)$, $p_t = (P_t^{x}, P_t^{\xi})$, $U_t(q_t) = \tilde{U}_t^{\theta}(X_t, \xi_t)$, and $K(q_t, p_t) = \frac{\beta(\xi_t)}{2M_x}\lVert P_t^x\rVert^2 + \frac{1}{2M_{\xi}}\lVert P_t^{\xi}\rVert^2$, and
\begin{equation}
    \Gamma_t \triangleq 
    \begin{bmatrix}
        \gamma_t^x I_d & 0\\
        0 & \gamma_t^{\xi}
    \end{bmatrix}
    \quad \quad \text{and} \quad \quad E_t \triangleq 
    \begin{bmatrix} 
    \varepsilon_t^x I_d & 0 \\
    0 & \varepsilon_t^{\xi}
    \end{bmatrix}.
\end{equation}
Thus, \cref{thm:ct_jarzynski} follows immediately from \cref{appendix:thm:jarzynski} in the special case that $h(X) = h(X_T)$ is a function of the terminal point.
\end{proof}
Now recall from \cref{eq:xz_theta_joint_density} that
\begin{equation}
    \hat{\pi}_t^{\theta}(x, \xi) = e^{-U_t^{\xi}(x) + F_t^\theta(\xi) + \psi_t'(\xi)} = e^{-\tilde{U}_t^{\theta}(x,\xi)},
\end{equation}
and define by $\pi_t^{\theta}(x, \xi)$ the normalized version of $\hat{\pi}_t^{\theta}(x, \xi)$. Additionally recall that $\PP_t$ and $\QQ_t$ denote the forward and backward path measures on the interval $[0,t]$, for $t \in [0,1]$. In particular, if $X \sim \QQ_t$, then $X_t \sim \pi_t^{\theta}$.  Then, \cref{thm:ct_jarzynski} may be applied to reweight the multi-temperature PINN objective from \cref{eq:ct_pinn} via
\begin{equation}
\begin{aligned}
    \mathcal{L}^{\text{MT}}_{\pinn}(F,\mu; \pi^\theta) &\triangleq \int_{\mathcal{T}} \mathbb{E}_{(x,z) \sim \pi^{\theta}}\left[\lvert  \partial_t F_t - \partial_t \tilde{U}_t^{\theta} - \nabla_x \cdot \mu_t + (\nabla_x \tilde{U}_t^{\theta})^T \mu_t \rvert^2\right]\d t\\
    &= \int_{\mathcal{T}}\mathbb{E}_{Q \sim \QQ_t}\left[\lvert  \partial_t F_t - \partial_t \tilde{U}_t^{\theta} - \nabla_x \cdot \mu_t + (\nabla_x \tilde{U}_t^{\theta})^T \mu_t \rvert^2\right]\d t\\
    &= \int_{\mathcal{T}} \frac{\mathbb{E}_{Q \sim \PP_t}\left[\lvert  \partial_t F_t - \partial_t \tilde{U}_t^{\theta} - \nabla_x \cdot \mu_t + (\nabla_x \tilde{U}_t^{\theta})^T \mu_t \rvert^2\exp(A_t(Q))\right]}{\mathbb{E}_{Q \sim \PP_t} \left[\exp(A_t(Q))\right]} \d t.
\end{aligned}
\end{equation}
yielding \cref{eq:reweighted_pinn}. To arrive at the expression for $A_t(Q)$ (from \cref{eq:work}) from the general work functional in \cref{appendix:eq:work}, we have used the fact the $\xi$-component of the control $\mu_t^\theta(q_t)$ is zero to obtain $\nabla_q \cdot \mu_t^\theta = \nabla_x \cdot \mu_t^\theta$, and this together with the fact that the kinetic energy depends only on the temperature and momenta to obtain that $(\mu_t^\theta)^T \nabla_q \mathcal{H}_t^\theta = (\mu_t^\theta)^T \nabla_x \tilde{U}_t^{\theta}$.

\section{Experimental Details}

\subsection{40-Mode Gaussian Mixture}
\label{appendix:experiment}
\paragraph{Additional Training Details.} At train time, we consider the four proposal types as outlined in \cref{fig:experiment}. For each, we parameterize the learned control $\mu_t^{\theta}(x)$ ($\mu_t^{\theta}(x,\beta(\xi))$ for CTDS), the free energy $F_t^{\theta}$ ($F_t^{\theta}(\beta(\xi))$ for CTDS), and the learned potential $U_t^{\theta}$ as feed-forward neural networks with width $256$ and depth three, and using the SiLU non-linearity \citep{silu}.  We train using a replay buffer, re-sampling once per epoch for 1250 epochs for a total of 125000 training iterations at which the PINN loss is evaluated at 6250 randomly sampled elements of the buffer. We utilize the Adam optimizer \citep{adam} with learning rate $1 \times 10^{-3}$, reducing by a factor of $\gamma = 0.97$ every $1000$ iterations, and after an initial burn-in period of $15000$ training iterations. Additionally, we follow the lead of \citep{expert_guide_pinn, nets} in utilizing curriculum-based training whereby at $T \in \{0.1, 0.2, 0.3, 0.4, 0.5, 0.6, 0.7, 0.8, 0.9\}$ (in that order) we spend $\{1000, 1000, 1000, 1000, 2000, 2000, 2000, 3000, 3000, 3000\}$ iterations respectively, and the remaining iterations at $T=1.0$. For CTDS, we reparameterize using \cref{appendix:eq:reparam} with $\beta_{\min} = 0.2$, $\Delta = 0.25$, and $\Delta' = 1.9$, and additionally set $\psi'_t(\xi) = \psi^{\text{conf}}(\xi)$ as in \cref{appendix:eq:confining} with $\eta = 10.0$ and $\tilde{\Delta} = 2.0$. Finally, we utilize Gaussian Fourier features as in \citep{tancik2020fourierfeaturesletnetworks} to encode position, time, and temperature with $100$, $20$, and $20$ features respectively, drawn from isotropic Gaussians with standard deviations $0.1$, $5$, and $1$, respectively.

\paragraph{Sampling Details.} In practice, we discretize using an Euler solver with with $\Delta t = 0.004$ (250 timesteps). Letting $\phi^{\theta}: \mathbb{R}^d \to \mathbb{R}^d$, $\phi^{\theta}: X_0 \mapsto X_1$ denote the corresponding unit-time flow of \cref{eq:experiment_ode}, we sample from the pushforward $\tilde{\pi}^{\theta} \triangleq [\phi^{\theta}]_{\sharp} \mathcal{N}(0, 5.0I_2)$. The density $\tilde{\pi}^{\theta}$ may then be computed via the continuous change of variables formula as
\begin{equation}
    \label{eq:change_of_variables}
    \log \tilde{\pi}^{\theta}(X_1) = \log \pi_0(X_0) - \int_0^1 \nabla \cdot \mu_t^{\theta}(X_t) \d t.
\end{equation}

\paragraph{Evidence Lower Bound.} Letting $\tilde{\pi}^{\theta}$ denote our sample density as given in \cref{eq:change_of_variables}, and $\hat{\pi}_1$ the unnormalized target, we define the evidence lower bound (ELBO) as 
\begin{equation}
    \mathbb{E}_{x \sim \tilde{\pi}^{\theta}}\left[\log \left(\frac{\hat{\pi}_1(x)}{\tilde{\pi}^\theta(x)}\right)\right] = -\dkl{\tilde{\pi}^\theta}{\pi_1} + \log Z \le \log Z,
\end{equation}
where $Z = \int \hat{\pi}_1(x)\, \d x$ denotes the partition function of $\hat{\pi}_1$.

\paragraph{Evidence Upper Bound.} Following \cite{beyondelbo}, we define the evidence upper bound (EUBO) as 
\begin{equation}
    \mathbb{E}_{x \sim \pi_1}\left[\log \left(\frac{\hat{\pi}_1(x)}{\tilde{\pi}^\theta(x)}\right)\right] = \dkl{\pi_1}{\tilde{\pi}^\theta} + \log Z \ge \log Z,
\end{equation}
where $Z$ is defined as before.

\begin{figure}
    \centering
    \centering
    \renewcommand{\arraystretch}{1.2}
    \begin{tabular}{|c|c|>{\centering\arraybackslash}p{4cm}|>{\centering\arraybackslash}p{4cm}|}
        \multicolumn{4}{c}{} \\
        \hline
        Name & Dynamics  & Hyperparameters & Density Path \\
        \hline
        Baseline & $\d X_t = \mu_t^{\theta}(X_t)\d t$ & N/A & $U_t = (1-t)U_0 + tU_1 + t(1-t)U_t^{\theta}$ \\
        \hline
        NETS (OD)  & \cref{eq:controlled_overdamped_langevin} & $\varepsilon_t = 50.0$ & $U_t = (1-t)U_0 + tU_1 + t(1-t)U_t^{\theta}$ \\
        \hline
        NETS (UD) & \cref{eq:controlled_underdamped_langevin} & $\gamma_t = 50.0$, $\varepsilon_t = 2.0$, $M = 1.0$ & $U_t = (1-t)U_0 + tU_1 + t(1-t)U_t^{\theta}$ \\
        \hline
        CTDS & \cref{eq:ctds} & $\gamma_t^x = 50.0$, $\varepsilon_t^x = 2.0$, $\gamma_t^{\xi} = 5.0$, $\varepsilon_t^{\xi} = 2.0$, $M_x = M_{\xi} = 1.0$ & $U_t^{\xi} = (1-t)U_0^{\xi} + tU_1^{\xi} + \beta(\xi)t(1-t)U_t^{\theta}$ \\
        \hline
    \end{tabular}
    \caption{The proposals and their accompanying density paths for the 40-mode Gaussian mixture experiment described in \cref{appendix:experiment}.}
    \label{fig:experiment}
\end{figure}

\end{document}